\documentclass{article}
\usepackage{amsmath}
\usepackage{amssymb}
\usepackage{amsthm}
\usepackage{booktabs}
\usepackage{xcolor}
\usepackage{enumitem}
\usepackage{graphicx}
\usepackage{caption}
\usepackage{bm}
\usepackage{amsfonts}

\usepackage{algorithm}
\usepackage{algpseudocode}
\usepackage{algorithmicx}
\usepackage{booktabs}
\usepackage{array}
\usepackage{multirow}  
\usepackage[colorlinks=true, linkcolor=blue, citecolor=blue, urlcolor=blue]{hyperref}

\usepackage{natbib}
\usepackage{graphicx}
\usepackage{subcaption}
\usepackage[most]{tcolorbox} 

\usepackage{bm}
\usepackage{caption}
\usepackage{xcolor}
\usepackage{authblk}

\newtheorem{theorem}{Theorem}[section]
\newtheorem{lemma}[theorem]{Lemma}

\usepackage{xr-hyper}                
\usepackage{hyperref}                 

\graphicspath{{Fig/}}
\theoremstyle{plain}
\newtheorem{assumption}{Assumption} 
\begin{document}
		\title{\textbf{Beyond Passive Aggregation: Active Auditing and Topology-Aware Defense in Decentralized Federated Learning}}
	
	\author[1]{Sheng Pan}
	\author[1]{Niansheng Tang\thanks{Corresponding author. Email: \texttt{nstang@ynu.edu.cn}}}
	\affil[1]{\small Yunnan Key Laboratory of Statistical Modeling and Data Analysis, Yunnan University, Kunming 650091, China}
	
	\date{} 
	
	\maketitle
	\begin{abstract}
Decentralized Federated Learning (DFL) remains highly vulnerable to adaptive backdoor attacks designed to bypass traditional passive defense metrics. To address this limitation, we shift the defensive paradigm toward a novel active, interventional auditing framework. First, we establish a dynamical model to characterize the spatiotemporal diffusion of adversarial updates across complex graph topologies. Second, we introduce a suite of proactive auditing metrics, stochastic entropy anomaly, randomized smoothing Kullback-Leibler divergence, and activation kurtosis. These metrics utilize private probes to stress-test local models, effectively exposing latent backdoors that remain invisible to conventional static detection. Furthermore, we implement a topology-aware defense placement strategy to maximize global aggregation resilience. We provide  theoretical property for the system's convergence under co-evolving attack and defense dynamics. Numeric empirical evaluations across diverse architectures demonstrate that our active framework is highly competitive with state-of-the-art defenses in mitigating stealthy, adaptive backdoors while preserving primary task utility.
	\end{abstract}
	
	\section{Introduction}
	Standard federated learning relies on a central server to aggregate local models, as seen in the \texttt{FedAvg} algorithm \citep{mcmahan2017communication}. However, this design suffers from communication bottlenecks, single points of failure, and the inherent risk of trusting a central authority. Decentralized Federated Learning (DFL) \citep{lian2017can} addresses these issues by distributing the aggregation process across participating nodes. While this peer-to-peer approach removes the central bottleneck, it also exposes the network to malicious participants. As a result, securing DFL requires Byzantine-robust consensus mechanisms that can filter out poisoned updates and maintain global model integrity.
	
	Existing literature primarily addresses these threats through consensus-based detection and robust aggregation methods. Early frameworks drew extensively from classical robust statistics; for instance, \textit{Krum} and \textit{Multi-Krum} \citep{blanchard2017machine} utilize  the Euclid distance to reject abnormal clients, while \textit{Trimmed Mean} \citep{Yin2018ByzantineRobustDL} and \textit{RSA} \citep{li2019rsa} employ coordinate-wise truncation and $\ell_1$-regularization. Subsequent approaches, such as \textit{FoolsGold} \citep{fung2020limitations}, leveraged directional statistics and cosine similarity to penalize correlated updates characteristic of sybil attacks. More recently, \textit{FLAME} \citep{nguyen2022flame} formulated malicious identification as an unsupervised clustering problem to isolate adversarial components from the empirical distribution.
	A  limitation of current static defensive paradigms is their reliance on deterministic spatial metrics, rendering them inherently transparent to adaptive adversaries. When a defensive metric is known, an attacker can explicitly minimize it within their objective function to mimic the statistical distribution of benign clients. This vulnerability has driven a rapid evolution of backdoor attacks in distributed learning, as extensively documented in recent literature.  \citet{bhagoji2019analyzing} and \citet{bagdasaryan2020backdoor} first demonstrated that adversaries could successfully compromise global models by injecting targeted malicious local updates. Building upon these initial vulnerabilities, \citet{wang2020attack} introduced edge-case backdoors, revealing that malicious parameters operating in the tails of the data distribution can maintain a highly stealthy profile. 
	\citep{zhang2022neurotoxin} introduce the Neurotoxin method.  Neurotoxin targets the least-frequently updated parameters using sparse masks to embed durable backdoors that resist being overwritten by benign updates while remaining stealthy to coordinate-based defenses.

	In contrast to passive filtering, the field of information security utilizes honeypots, decoy resources designed to compel adversaries into revealing their intent through interaction. This philosophy of ``active intervention'' has surfaced in computer vision ~\citep{wang2019neural, gao2019strip}, which utilize intentional perturbations to expose stimulus-response inconsistencies in trojaned models. However, these methods remain tailored for centralized visual tasks and localized pixel-space triggers. We propose to extend this paradigm to DFL by injecting stochastic, private probes into the auditing process. This creates a critical information asymmetry, forcing malicious nodes to reveal latent functional anomalies that remain invisible to static methods.
	
	Furthermore, there is a distinct scarcity of system-level studies investigating the temporal diffusion of adversarial behaviors. Most research is confined to neutralizing immediate threats at the point of aggregation, ignoring the fact that DFL is essentially a consensus optimization process over Markov networks. Similar to the Kermack-McKendrick model~\citep{Kermack1927ACT} in epidemiology, understanding the diffusion dynamics is critical for following reasons:
	(1) In DFL networks, malicious updates can propagate several hops before detection. Modeling this is essential for determining optimal defense frequency.
	(2)  DFL relies on complex graphs (e.g., scale-free or random regular graphs), where connectivity dictates the speed and reach of poisoning.

	To bridge these gaps, we shift the defensive paradigm from localized, passive aggregation to a system-wide, interventional auditing framework. Our contributions are as follows:

	(i) Rather than treating Byzantine attacks as independent anomalies, we formulate a dynamical system to capture the spatiotemporal spread of malicious updates. This approach uncovers how contamination cascades hierarchically, offering a theoretical basis for evaluating defense latency across complex network structures.
	
	(ii)  We introduce a multi-scale diagnostic pipeline comprising three novel metrics: stochastic entropy anomaly ($\rho_{SEA}$), randomized smoothing KL-Divergence ($\rho_{RS}$), and asymmetric element-wise Z-Score ($\rho_{\mathcal{A}_K}$). By interrogating local models with private, randomized probes, we evaluate their functional consistency and output stability. This interventional approach exploits the information asymmetry between defenders and adversaries, thereby constraining the optimization space available for stealthy evasion. These diagnostic scores are then integrated into a Multi-Armed Bandit (MAB) framework, allowing defending nodes to dynamically evaluate neighbor reliability and prioritize high-trust updates during aggregation. 
	
	(iii) To address the computational constraints of active auditing in large-scale networks, we introduce a defense allocation mechanism tailored to the underlying graph structure. We design specific deployment strategies for scale-free, random-regular, and grid networks to ensure defense nodes occupy critical topological hubs and bottlenecks. Consequently, this strategic placement allows the system to achieve global resilience by equipping only a small fraction of participating nodes with auditing capabilities.
	
	(iv) Utilizing a novel separable principle, we provide a theoretical framework of the DFL algorithm under co-evolving defense and attack. We characterize convergence rates in non-convex regimes and quantify the trade-off between robustness and network spectral properties. 
	
	(v)  Through experimental validation, we characterize the relationship between our proposed theoretical diffusion bound and the attack success rate (ASR). Interestingly, we observe a structural dichotomy dictated by the target model architecture: while attacks on high-dimensional models (e.g., Transformers) exhibit a smooth, continuous escalation in ASR, attacks targeting spatial models (e.g., CNNs) demonstrate a distinct phase transition.  Furthermore, we evaluate the proposed defense through comparative benchmarking against state-of-the-art methods.

The remainder of this paper is organized as follows: Section \ref{sec:preliminaries} formalizes the system model and preliminaries for decentralized federated learning. Section \ref{sec:diffu} establishes the dynamical model for attack diffusion to characterize the spatiotemporal spread of malicious updates. Section \ref{sec:metrics} introduces our multi-scale proactive auditing metrics designed to expose latent backdoors. Section \ref{sec:4} details the proposed defense strategy, encompassing the MAB-guided robust aggregation protocol, topology-aware placement, and theoretical convergence analysis. Section 6 presents comprehensive empirical evaluations across different architectures and graph topologies. Finally, Section 7 concludes the paper and discusses the inherent trade-offs in extreme non-IID settings.	 Detailed proofs are provided in the Appendix.
	\section{System Model and Preliminaries}
	\label{sec:preliminaries}
	\subsection{Notations}
	Throughout this paper, we define the following notations.	$\otimes$ denotes the Kronecker product. For any matrices $\mathbf{A} \in \mathbb{R}^{m \times n}$ and $\mathbf{B} \in \mathbb{R}^{p \times q}$, $\mathbf{A} \otimes \mathbf{B}$ is the $mp \times nq$ block matrix $[a_{ij} \mathbf{B}]$.  $\text{vec}(\cdot)$ denotes the vectorization operator that transforms a tensor into a column vector. 
	$\mathbf{I}_{d}$ represents the $d \times d$ identity matrix.
	$\|\cdot\|$ denotes the standard Euclidean norm ($\ell_2$-norm) for vectors, defined as $\|\mathbf{v}\| = \sqrt{\sum v_i^2}$. Let $\lambda_i(\mathbf{A})$ denote the $i$-th largest eigenvalue of a square matrix $\mathbf{A}$ such that $\lambda_{\max}(\mathbf{A}) \geq \lambda_2(\mathbf{A}) \geq \dots \geq \lambda_{\min}(\mathbf{A})$. For a matrix $\mathbf{A} \in \mathbb{R}^{m \times n}$, $\|\mathbf{A}\|= \sqrt{\lambda_{\max}(\mathbf{A}^\top \mathbf{A})}$ represents the spectral norm.  The notation $\lceil x \rceil$ denotes the ceiling function, which maps a real number $x$ to the smallest integer greater than or equal to $x$. The $\operatorname{softmax}: \mathbb{R}^K \to \Delta^{K-1}$ operator transforms a $K$-dimensional vector of real-valued logits $\bm{h}$ into a probability distribution on the $(K-1)$-dimensional simplex, denoted as $\Delta^{K-1}$. For each component $k \in \{1, \dots, K\}$, the output is defined as:$\operatorname{softmax}(\bm{h})_k = \frac{\exp(h_k)}{\sum_{j=1}^K \exp(h_j)}.$  $\text{Median}(\cdot)$ and $\text{MAD}(\cdot)$ denote the empirical median and the median absolute deviation, respectively. For a finite set of scalars $\mathcal{A}$, the latter is defined as:$$\text{MAD}(\mathcal{A}) = \text{Median}(\{ |x - \text{Median}(\mathcal{A})| : x \in \mathcal{A} \}).$$
	\subsection{System Model}
	
We consider a supervised learning task where $\mathcal{X} \subseteq \mathbb{R}^{d_{in}}$ denotes the input feature space with dimensionality $d_{in}$ (representing raw observations such as flattened pixels or token embeddings), and $\mathcal{Y} = \{1, 2, \dots, K\}$ represents the target space for a $K$-class classification problem. Each node $i \in \{1, \dots, n\}$ possesses a local private dataset $\mathcal{S}_i$, where samples are drawn from a potentially non-identical local distribution.
	
	Suppose the model parameters are structured as a collection of tensors $\{\mathbf{\Theta}^{(l)}\}_{l=1}^L$ corresponding to $L$ layers. For analytical convenience, we denote the global parameter vector as $\bm{\theta} = [\text{vec}(\mathbf{\Theta}^{(1)})^\top, \dots, \text{vec}(\mathbf{\Theta}^{(L)})^\top]^\top \in \mathbb{R}^p$. 
	Let $\bm{r}(\bm{x}) \in \mathbb{R}^{d_{prev}}$ denote the intermediate representation extracted by the preceding $L-2$ layers, where $d_{prev}$ denotes its dimensionality. The target hidden layer $\mathbf{\Theta}^{(L)}$ is parameterized by a weight matrix $\mathbf{U} \in \mathbb{R}^{d \times d_{prev}}$ and a bias vector $\mathbf{b}_{L-1} \in \mathbb{R}^d$. The feature extractor output, which corresponds to the terminal latent representation, is formulated as:
	$
	\bm{z} = \bm{\phi}_{\bm{\theta}}(\bm{x}) := \psi(\mathbf{U} \bm{r}(\bm{x}) + \mathbf{b}_{L-1}) \in \mathbb{R}^d,
	$
	where $\psi(\cdot)$ is a non-linear activation function (e.g., ReLU). Each individual scalar feature in this space is given by the inner product $z_j = \psi(\langle \bm{u}_j, \bm{r}(\bm{x}) \rangle + b_{L-1, j})$, where $\bm{u}_j$ is the $j$-th row of $\mathbf{U}$.
	Finally, the terminal layer $\mathbf{\Theta}^{(L)}$ (the classification head) is parameterized by $\mathbf{V} \in \mathbb{R}^{K \times d}$ and $\mathbf{b} \in \mathbb{R}^K$. The logit output is:
	$
	\bm{h}_{\bm{\theta}}(\bm{x}) = \mathbf{V} \bm{\phi}_{\bm{\theta}}(\bm{x}) + \mathbf{b}.
	$

	To convert these raw logits into a prediction probability distribution over the target classes, we apply the softmax operator $\operatorname{softmax}: \mathbb{R}^K \to \Delta^{K-1}$. The resulting predictive probability distribution is denoted by $\bm{p}(\bm{x}, \bm{\theta}) := \operatorname{softmax}(\bm{h}_{\bm{\theta}}(\bm{x}))$, where the probability assigned to the $k$-th class is given by:
	$$
	p(\bm{x}, \bm{\theta})_k = \frac{\exp(\bm{h}_{\bm{\theta}}(\bm{x})_k)}{\sum_{j=1}^K \exp(\bm{h}_{\bm{\theta}}(\bm{x})_j)}.
	$$
	The primary objective of DFL is to identify a global minimizer that minimizes the aggregate empirical risk across $n$ distributed nodes. In this decentralized paradigm, nodes interact with their immediate topological neighbors. Formally, we define the global objective function $f(\bm{\theta})$ as:
	$$
	\min_{\bm{\theta} \in \mathbb{R}^p} F(\bm{\theta}) := \frac{1}{n} \sum_{i=1}^n f_i(\bm{\theta}),
	$$
	where  $f_i$ is the local loss function for node $i$.
	
	In the decentralized setting, no central server exists. We model the communication topology as an undirected graph $\mathcal{G} = (\mathcal{V}, \mathcal{E})$, where $\mathcal{V} = \{1, 2, \dots, n\}$ is the set of participating nodes and $\mathcal{E} \subseteq \mathcal{V} \times \mathcal{V}$ represents the set of available communication links. An edge $(i, j) \in \mathcal{E}$ indicates that node $i$ and node $j$ can directly exchange information. 
	For a specific node $i \in \mathcal{V}$, its model parameters at communication round $t$ are partitioned into $L$ functional layers. We denote the vectorized parameters of the $l$-th layer for node $i$ as $\bm{\theta}_{i, l}^{(t)}$. Correspondingly, the global model update for node $i$ at round $t$ can be expressed as the concatenation of these layer-wise sub-vectors: $\bm{\theta}_i^{(t)} = [(\bm{\theta}_{i, 1}^{(t)})^\top, (\bm{\theta}_{i, 2}^{(t)})^\top, \dots, (\bm{\theta}_{i, L}^{(t)})^\top]^\top$.
	Optimization is instead performed through a continuous sequence of local gradient steps and consensus-based aggregations. Specifically, we consider the standard configuration where each node performs exactly one local gradient descent step per communication round. At each iteration $t$, node $i$ updates its local model $\bm{\theta}_i^{(t)}$ by interacting solely with its immediate neighbors $\mathcal{N}_i = \{j \mid (i, j) \in \mathcal{E}\}$:
$$\bm{\theta}_i^{(t+1)} = \sum_{j \in \mathcal{N}_i \cup \{i\}} w_{ij} \bm{\theta}_j^{(t)} + \Delta \bm{\theta}_i^{(t)}$$
	where $\mathbf{W} = [w_{ij}] \in \mathbb{R}^{n \times n}$ is a doubly stochastic mixing matrix, $\bm{\theta}_j^{(t)} $ is the update of $j$-th client at $t$-th round, Where $\Delta \bm{\theta}_i^{(t)}$ represents the cumulative model update obtained after executing $E$ local training epochs on client $i$, and $\eta > 0$ is the learning rate.

\section{Attack Diffusion Dynamics}\label{sec:diffu}

The Kermack-McKendrick model \citep{Kermack1927ACT} describes the spread of epidemics by dividing populations into susceptible, infectious, and recovered states. The propagation of malicious attacks in DFL follows a similar process: benign nodes act as the susceptible population, injected malicious updates are the infection, and local training, which naturally overwrites task-irrelevant backdoor weights, serves as the recovery mechanism.

Building on this analogy, we formulate a discrete-time dynamical system to model how adversarial gradients spread across the network. Unlike standard epidemic models that assume uniform mixing, our approach explicitly accounts for the specific graph topology and the discrete communication rounds in DFL. Using this model, we establish analytical bounds on network contamination and calculate the stationary distribution of the attack intensity.

\subsection{System Dynamics of Infection Intensity}
We model the evolution of the global infection intensity vector $\boldsymbol{s}(t) = [s_1(t), \ldots, s_n(t)]^T$ through the interaction of three primary mechanisms:

1. \textbf{Network Mixing:} In DFL, nodes update their parameters by aggregating models from their neighbors. This consensus step naturally dilutes the backdoor's strength. Governed by the row-stochastic mixing matrix $\mathbf{W}$, this step acts as a Markovian transition operator: $\boldsymbol{s}_{\text{agg}} = \mathbf{W}\boldsymbol{s}(t)$.

2. \textbf{Signal Attenuation:} During local training, benign clients optimize their models for the primary task. Because the backdoor trigger is structurally independent of the main task's feature distribution, standard gradient descent treats the adversarial weights as unstructured noise, leading to gradual decay. While the true degradation of backdoor parameters in deep, non-convex neural networks is a complex non-linear process, we adopt a linear first-order approximation for theoretical tractability. We model this baseline attenuation as $s_i(t+1) \leftarrow (1 - \lambda_i)s_i(t)$, where the decay rate $\lambda_i \in (0, 1)$ applies to benign nodes, and $\lambda_i = 0$ for malicious nodes. This defines the diagonal attenuation matrix $\mathbf{\Lambda} = \text{diag}(\lambda_1, \ldots, \lambda_n)$.

3. \textbf{Adversarial Injection:} To counteract network mixing and attenuation, adversaries actively inject malicious gradients. This introduces an additive constant term: $s_i(t+1) \leftarrow (1 - \lambda_i)s_i(t) + u_i$, where the injection magnitude satisfies $u_i > 0$ for malicious nodes and $0$ otherwise, forming the vector $\mathbf{u}$.

Synthesizing these discrete mechanisms, we establish a discrete-time dynamical system for the infection intensity:
\begin{equation}\label{eq:diffusion}
	\boldsymbol{s}(t+1) = \underbrace{(\mathbf{I} - \mathbf{\Lambda})}_{\text{Decay}} \underbrace{\mathbf{W}}_{\text{Diffusion}} \boldsymbol{s}(t) + \underbrace{\mathbf{u}}_{\text{Injection}}.
\end{equation}
Let $\mathbf{A} = (\mathbf{I} - \mathbf{\Lambda})\mathbf{W}$ be the effective transition matrix. The asymptotic accumulation of the infection intensity is governed by the spectral radius $\rho(\mathbf{A})$ and the source vector $\mathbf{u}$. Assuming the system reaches a stationary equilibrium $\boldsymbol{s}(t+1) = \boldsymbol{s}(t) = \boldsymbol{s}^*$, the asymptotic distribution is given by:
$$
\boldsymbol{s}^* = (\mathbf{I} - \mathbf{A})^{-1}\mathbf{u}.
$$
This analytical solution provides clear structural insights. The resolvent operator $(\mathbf{I} - \mathbf{A})^{-1}$ indicates that the expected contamination of any specific node is not uniform; rather, it is  dependent on its shortest-path distance to the adversarial nodes and the cumulative decay rate $\lambda$ along those propagation paths.

We establish the spatiotemporal boundaries of the infection assuming a clean initialization $\boldsymbol{s}(0) = \mathbf{0}$.

\begin{lemma}[System Stability]\label{lem:1}
	Assume the network graph induced by $\mathbf{W}$ is strongly connected and there exists at least one benign node $i$ with a  positive decay rate $\Lambda_{ii} = \lambda > 0$. Then, the spectral radius satisfies $\rho(\mathbf{A}) < 1$, guaranteeing the existence of the stationary state $\boldsymbol{s}^*$ and the transient solution $\boldsymbol{s}(t) = (\mathbf{I} - \mathbf{A}^{t}) \boldsymbol{s}^*$.
\end{lemma}

\begin{lemma}[Spatiotemporal Diffusion Bound]\label{cor:1}
	Under the conditions of Lemma \ref{lem:1}, the backdoor intensity $s_i(t)$ of a benign node $i$ at communication round $t$ is bounded by:
	$$
	s_i(t) \leq u_s \cdot \underbrace{(1-\lambda)^{d_{is}}}_{\text{Geometric Attenuation}} \cdot \underbrace{\sum_{k=0}^{t-d_{is}} (1-\lambda)^k}_{\text{Temporal Accumulation}},
	$$
where $u_s$ is the injection magnitude at the adversarial source $s$, and $d_{is}$ is the shortest path distance from node $i$ to $s$.
\end{lemma}

This result highlights the localized nature of the update process: for $t < d_{is}$, the summation is empty, reflecting the one-hop-per-round transmission constraint. Over time, the injected signal decays geometrically with respect to the spatial distance while accumulating asymptotically across iterations.

The theoretical framework of Lemma \ref{lem:1} allows us to characterize the spatial distribution of the contamination. By expanding the stationary solution via the Neumann series, we partition the global intensity by path length:
$$
\boldsymbol{s}^* = \left( \sum_{k=0}^{\infty} \mathbf{A}^k \right) \mathbf{u} = \underbrace{\mathbf{u}}_{\text{Order 0}} + \underbrace{\mathbf{A}\mathbf{u}}_{\text{Order 1}} + \underbrace{\mathbf{A}^2\mathbf{u}}_{\text{Order 2}} + \dots
$$

Since the spectral radius satisfies $\rho(\mathbf{A}) < 1$, the operator $\mathbf{A}^k$ induces a geometric decay rate of approximately $\mathcal{O}(((1-\lambda)/\bar{d})^k)$, where  $\bar{d}$ denotes the mean node degree of the network graph. This expansion reveals that contamination is primarily concentrated in the immediate neighborhood of the source. The zeroth-order term $\mathbf{u}$ represents the initial malicious injection. The first-order term $\mathbf{A}\mathbf{u}$ affects the direct neighbors, where nodes with lower degrees experience higher contamination variance due to a reduced mixing capacity.
\section{Backdoor Detection Metrics}\label{sec:metrics}
Defending against backdoor attacks is challenging because compromised models exhibit malicious functional divergence while maintaining  parametric similarity to benign updates. Let $\mathcal{M}$ denote the set of malicious nodes. An adversary at node $i \in \mathcal{M}$ seeks an optimal poisoned parameter vector $\bm{\theta}^*$ by minimizing a composite objective function:

$$
	\min_{\bm{\theta}^* \in \mathbb{R}^p} \quad \sum_{(x,y) \in \mathcal{D}_i}\underbrace{ \ell(h_{\bm{\theta}^*}(\bm{x}), y)}_{\text{Main Task Fidelity}} + \kappa_0 \cdot \underbrace{ \ell(h_{\bm{\theta}^*}(x \oplus \bm{\delta}), y_t)}_{\text{Backdoor Installation}} + \kappa_1 \cdot \underbrace{\mathcal{R}(\bm{\theta}^*; \Omega)}_{\text{Evasion Penalty}},
$$
where $\ell$ is the standard loss function, and $\mathcal{R}(\bm{\theta}^*; \Omega)$ represents the evasion penalty designed to constrain the model within a valid region $\Omega$ permitted by the system's detection mechanisms. Here, $\bm{\delta}$ represents the backdoor trigger, and the operator $\oplus$ denotes a generalized trigger injection mechanism (e.g., additive perturbation or patch replacement). Consequently, $x \oplus \bm{\delta} \in \mathcal{X}$ represents the poisoned input. The hyperparameters $ \kappa_0 > 0$ and $ \kappa_1 > 0$ control the trade-off between main-task accuracy (ACC), attack success rate (ASR), and evasion stealth.

Previous studies \citep{bhagoji2019analyzing,bagdasaryan2020backdoor,wang2020attack,zhang2022neurotoxin} demonstrate that standard anomaly detection mechanisms are highly susceptible to distributional mimicry. Because conventional robust aggregators rely on deterministic detection statistics, the evasion penalty $\mathcal{R}$ can be explicitly instantiated as an optimization constraint. Depending on the defensive landscape of the decentralized network, adversaries typically employ the following evasion strategies:

\paragraph{Case A: Distance Evasion} 
To bypass distance-based outlier detection (e.g., Krum \citep{blanchard2017machine}), an adversary at node $i$ bounds the malicious parameter $\bm{\theta}^*$ near the local consensus. This geometric constraint ensures that the submitted update remains indistinguishable from the neighborhood trajectory in terms of Euclidean distance:
$$
\mathcal{R}(\bm{\theta}^*) = \left\| \bm{\theta}^* - \bm{\theta}_{ref}  \right\|^2,
$$
where $\bm{\theta}_{ref}$ is a benign reference parameter vector.

\paragraph{Case B: Directional Evasion}  
To evade defenses that evaluate the angular consistency of model updates (e.g., \textbf{FoolsGold} \citep{fung2020limitations}, \textbf{FLAME} \citep{nguyen2022flame}, or \textbf{CosL2} \citep{bhagoji2019analyzing}), the adversary must geometrically align the direction of the malicious update with a benign reference trajectory. The penalty is formulated by minimizing the cosine distance:
$$
\mathcal{R}(\bm{\theta}^*) = 1 - \frac{\langle \bm{\theta}^* , \bm{\theta}_{ref} \rangle}{\| \bm{\theta}^* \| \cdot \| \bm{\theta}_{ref} \|}.
$$
This structural constraint ensures that the adversarial gradient perfectly mimics the benign angular distribution.

To address the predictable weaknesses of conventional metrics, we propose an active auditing mechanism. Unlike static filters, our approach actively tests local models using randomized or private inputs. We exploit the fact that malicious models, unlike benign ones, reveal abnormal predictions and activation patterns when processing these secret probes. By keeping the exact test parameters hidden from the adversary, our framework makes targeted evasion practically impossible.
\subsection{Stochastic Entropy Anomaly ($\rho_{SEA}$)}

To detect sophisticated adversaries that employ optimization constraints to evade parameter-space filtering, we introduce the stochastic entropy anomaly (SEA). This metric builds upon the principles of the STRIP method \citep{gao2019strip, wu2022backdoorbench}, extending its logic to the iterative auditing of local updates. 

Let $\bm{h}_{\bm{\theta}}(\bm{x})$ denote the output logits of a model parameterized by $\bm{\theta}$ for an input probe $\bm{x} \in \mathcal{X}$, where the predicted probability simplex is defined as $\bm{p}(\bm{x}, \bm{\theta}) = \operatorname{softmax}(\bm{h}_{\bm{\theta}}(\bm{x}))$. As defined, the predictor is composed of a feature extractor $\bm{\phi}_{\bm{\theta}}: \mathcal{X} \to \mathbb{R}^d$ and a linear projection $\mathbf{V} \in \mathbb{R}^{K \times d}$. For an unstructured stochastic probe $\bm{x}_{trap} \in \mathcal{X}$ (e.g., isotropic random noise), the resulting feature vector $\bm{\phi}_{\bm{\theta}}(\bm{x}_{trap})$ can be modeled as a random variable in a high-dimensional space. According to the concentration of measure on the unit sphere~\citep{vershynin2018high}, for any fixed class weight vector $\mathbf{v}_k \in \mathbb{R}^d$, the probability that a random feature $\bm{\phi}_{\bm{\theta}}(\bm{x}_{trap})$ has a non-negligible projection onto $\mathbf{v}_k$ vanishes as $d \to \infty$. Specifically, if $\bm{\phi}_{\bm{\theta}}(\bm{x}_{trap})$ is sampled from a distribution with zero mean and identity covariance, the expected inner product satisfies:
$$
\mathbb{E}[\langle \mathbf{v}_k, \bm{\phi}_{\bm{\theta}}(\bm{x}_{trap}) \rangle] \approx 0.
$$
Consequently, the logits $\bm{l} = \bm{h}_{\bm{\theta}}(\bm{x}_{trap})$ exhibit a near-uniform distribution. The resulting predictive distribution approaches the barycenter of the probability simplex $\Delta^{K-1}$, thereby maximizing the Shannon entropy:
$$
H(P) = -\sum_{k=1}^K p_k \log p_k \approx \log K,
$$
where $p_k := [\bm{p}(\bm{x}_{trap}, \bm{\theta})]_k$ denotes the predicted probability assigned to the $k$-th class.

Conversely, the parameter space of a compromised model can be conceptually partitioned into benign parameters and malicious sub-networks. For the attack to be effective, the weights in the malicious subspace must be sufficiently large to override the benign signal whenever the trigger is present. Consequently, even a slight perturbation near the trigger subspace forces the output probability mass to concentrate on the target class $y_t$, resulting in an anomalously low Shannon entropy.

Within the DFL context, however, predictive entropy often suffers from high variance due to client data heterogeneity, which can obscure the boundary between benign updates and stealthy malicious ones. To mitigate this noise and ensure reliable detection, we refine the metric into a normalized, worst-case index. Instead of relying on a single observation, we interrogate the model with a diverse set of stochastic exogenous probes, denoted as $\mathcal{X}_{trap} \cup \{-\mathcal{X}_{trap}\}$, designed to expose functional anomalies across different structural dimensions of the input space:
$$
\rho_{SEA} = \max_{\bm{x} \in \mathcal{X}_{trap} \cup \{-\mathcal{X}_{trap}\}} \left[ C_{max}(\bm{x}) \cdot \left( 1 - H_{norm}(\bm{x}) \right) \right],
$$
where 
$$
C_{max}(\bm{x}) = \max_{k \in \{1, \dots, K\}} p_k(\bm{x}) \in [1/K, 1],\quad
H_{norm}(\bm{x}) = \frac{-\sum_{k=1}^K p_k \log p_k}{\log K} \in [0, 1].
$$
To ensure the probes effectively explore the functional landscape, the distribution of stochastic test probes $\mathcal{X}_{trap}$ must be structurally aligned with the data modality; concrete examples are provided in the experiments section.

\subsection{Randomized Smoothing Kullback-Leibler (KL) Divergence ($\rho_{RS}$)}

While $\rho_{SEA}$ effectively captures the global rigidity of  backdoors, highly adaptive adversaries may attempt to circumvent this filter via entropy regularization. However, the efficacy of our defense architecture is anchored in information asymmetry. Because the specific stochastic test probes used in the $\rho_{SEA}$ auditing phase are stochastically sampled from a private distribution maintained solely by the defending nodes, the adversary operates completely blind to the exact coordinates of the test vectors. Faced with this intractable uncertainty in a high-dimensional space, the attacker is forced to apply a universal penalty across the entire stochastic perturbation space, coercing the model to yield a near-uniform, high-entropy predictive distribution for any random noise it encounters.

To counter this evasion strategy, we introduce the randomized smoothing KL divergence metric. To simultaneously satisfy the attacker's conflicting objectives, yielding high-entropy predictions for random noise while maintaining low-entropy targeted predictions for triggered data, an adaptive adversary must induce a sharp decision boundary  immediately outside the support of the clean data distribution. Consequently, applying a small perturbation to the malicious latent representation leads to high-entropy chaos.

Conversely, benign architectures naturally learn smooth decision boundaries and satisfy local Lipschitz continuity. For a small perturbation $\boldsymbol{\epsilon}$, the predictive output remains within a stable neighborhood of the original distribution. This functional stability ensures that the local decision geometry is preserved, yielding consistent categorical outputs~\citep{cohen2019certified,nakkiran2021deep}. 

We evaluate this local Lipschitz continuity by injecting an isotropic Gaussian perturbation $\boldsymbol{\epsilon} \sim \mathcal{N}(0, \sigma_{rs}^2 \mathbf{I})$ into the latent space, generating a perturbed prediction $\bm{P}_{noisy} := \operatorname{softmax}(\bm{h}_{\bm{\theta}}(\bm{z}+\boldsymbol{\epsilon}))$, where $\bm{z} = \bm{\phi}_{\bm{\theta}}(\bm{x})$ is the clean latent feature vector. To quantify this distributional shift, we utilize the KL Divergence:
$$
D_{KL}(\bm{P}_{clean} \parallel \bm{P}_{noisy}) = \sum_{i=1}^K P_{clean, i} \log \frac{P_{clean, i}}{P_{noisy, i}},
$$
where $\bm{P}_{clean} := \operatorname{softmax}(\bm{h}_{\bm{\theta}}(\bm{z}))$.

To integrate this sensitivity into a bounded anomaly scoring framework, let $\mathcal{E} = \{\boldsymbol{\epsilon}^{(1)}, \boldsymbol{\epsilon}^{(2)}, \dots, \boldsymbol{\epsilon}^{(B)}\}$ be a set of $B$ independent and identically distributed noise realizations. For a candidate parameter vector $\bm{\theta}$ evaluated during aggregation, we define the sequence of perturbed distributions as $\bm{P}_{noisy}^{(m)} = \operatorname{softmax}(\bm{h}_{\bm{\theta}}(\bm{z} + \boldsymbol{\epsilon}^{(m)}))$ for each $m \in \{1, \dots, B\}$. The Randomized Smoothing score ($\rho_{RS}$) is defined as the exponentially decayed empirical mean of the KL divergence over the sample set:
\begin{equation}\label{eq:KL}
	\rho_{RS} = \exp \left( -\frac{\kappa_{RS}}{B} \sum_{m=1}^{B} D_{KL}(\bm{P}_{clean} \parallel \bm{P}_{noisy}^{(m)}) \right),
\end{equation}
where $\kappa_{RS} > 0$ calibrates the decay rate.

\subsection{Activation Kurtosis ($\rho_{\mathcal{A}_K}$)}

While $\rho_{SEA}$ and $\rho_{RS}$ effectively evaluate  the predictive output space, they inherently treat the model's intermediate representations as a black box. To complement these output-level metrics, we introduce activation kurtosis ($\rho_{\mathcal{A}_K}$), which shifts the analytical focus from terminal predictions to the latent activation dynamics of individual neurons.

During the active auditing phase, a defender evaluates the model on a clean batch of $N$ independent input $\mathcal{X}_{probe} = \{\bm{x}^{(1)}, \dots, \bm{x}^{(N)}\}$. For each input $i$, let $\bm{r}^{(i)} \in \mathbb{R}^{d_{prev}}$ denote the input vector $\bm{x}^{(i)}$ to the target layer $\mathbf{\Theta}^{(L)}$. Recall that the layer $\mathbf{\Theta}^{(L)}$ is parameterized by the weight matrix $\mathbf{U} = [\bm{u}_1, \dots, \bm{u}_d]^\top \in \mathbb{R}^{d \times d_{prev}}$, where $\bm{u}_i$ represents the coefficient vector for the $i$-th feature dimension ($i \in \{1, \dots, d\}$). Ignoring the intercept term, the activation in the $i$-th dimension for observation $n$ is given by the linear combination $z_i^{(n)} = \bm{u}_i^\top \bm{r}^{(n)}$. The resulting $d$-dimensional feature vector is $\bm{z}^{(n)} = (z_1^{(n)}, \dots, z_d^{(n)})^\top \in \mathbb{R}^d$.

In a benign model, due to $\ell_2$-regularization (weight decay) during optimization, the activations $z_i^{(n)}$ are uniformly bounded by a constant $c_z$. Consequently, the empirical fourth moment, averaged across both the $d$ dimensions and $N$ observations, is bounded by the baseline capacity:
\begin{equation}\label{eq:kben}
	\frac{1}{N} \sum_{n=1}^N \left( \frac{1}{d} \sum_{i=1}^d \left(z_i^{(n)}\right)^4 \right) \le c_z^4.
\end{equation}

Conversely, suppose the adversary partitions the target layer's neurons such that the set of malicious neurons $\mathcal{S}_{M}$ is non-empty.  Recall that $\bm{r}(\bm{x}) \in \mathbb{R}^{d_{prev}}$ denotes the intermediate representation extracted by the preceding layers for a given input $\bm{x} \in \mathcal{X}$. When a backdoor trigger $\bm{\delta}$ is injected into the input space (yielding a poisoned input $\tilde{\bm{x}} = \bm{x} \oplus \bm{\delta}$), it induces a corresponding latent deviation $\bm{\delta}_u \in \mathbb{R}^{d_{prev}}$, where $\bm{\delta}_u = \bm{r}(\bm{x} \oplus \bm{\delta}) - \bm{r}(\bm{x}).$ To evade geometric detection in decentralized aggregators, this latent trigger must remain exceptionally stealthy, bounded by $\|\bm{\delta}_u\| \le \epsilon \to 0$.

When the trigger is present, the latent deviation induced on the $j$-th target neuron is bounded by the Cauchy-Schwarz inequality:
$$
|\Delta z_j| = |\langle \bm{u}_j, \bm{\delta}_u \rangle| \le \|\bm{u}_j\| \|\bm{\delta}_u\| \le \|\bm{u}_j\| \epsilon.
$$

Let $\bm{h} \in \mathbb{R}^K$ denote the benign logit vector for a clean input $\bm{x}$, and let $y_t$ denote the adversary's target class. Under the presence of the latent trigger $\bm{\delta}_u$, the poisoned latent representation $\tilde{\bm{z}} = \bm{z} + \Delta\bm{z}$ yields a perturbed logit vector $\tilde{\bm{h}} = \bm{h} + \Delta\bm{h}$, where $\Delta\bm{h} = \mathbf{V} \Delta\bm{z}$.

Because the $\operatorname{softmax}$ function is monotonic, achieving a targeted classification flip ($y_t = \arg\max_k \tilde{h}_k$) requires the poisoned target logit $\tilde{h}_{y_t}$ to  dominate all competing logits. This necessitates a minimum required decision margin $C > 0$:
$$
\Delta h_{y_t} > \max_{k \neq y_t} \tilde{h}_k - h_{y_t} := C.
$$

Substituting the logit perturbation, we establish the lower bound for the malicious weights:
$$
|\mathcal{S}_{M}| \cdot V_{max} \cdot \left(\max_{j \in \mathcal{S}_{M}} \|\bm{u}_j\|\right) \cdot \epsilon \ge \sum_{j \in \mathcal{S}_{M}} v_{y_t, j} \Delta z_j = \Delta h_{y_t} \ge C,
$$
where $V_{max} = \max_k |v_{y_t, k}|$. This  implies that $\max_{j \in \mathcal{S}_{M}} \|\bm{u}_j\| \ge \frac{C}{|\mathcal{S}_{M}| V_{max} \epsilon} = \mathcal{O}(1/\epsilon)$. 

Let $\hat{\mathbf{\Sigma}}_u = \frac{1}{N} \sum_{n=1}^N \bm{r}^{(n)} (\bm{r}^{(n)})^\top$ denote the empirical second-moment matrix of the clean intermediate features. Assuming the sampled benign feature space is diverse and non-degenerate, $\hat{\mathbf{\Sigma}}_u$ is positive definite with a minimum eigenvalue $\hat{\lambda}_{min} > 0$. By the properties of the Rayleigh quotient, the empirical average of the squared magnitude for the malicious activation satisfies:
$$
\frac{1}{N} \sum_{n=1}^N (z_j^{(n)})^2 = \bm{u}_j^\top \hat{\mathbf{\Sigma}}_u \bm{u}_j \ge \hat{\lambda}_{min} \|\bm{u}_j\|^2.
$$
Applying Jensen's inequality, we obtain the lower bound for the fourth moment:
\begin{equation}\label{eq:kmal}
	\max_{j\in\mathcal{S}_{M}}	\frac{1}{N} \sum_{n=1}^N (z_j^{(n)})^4 \ge \left( \frac{1}{N} \sum_{n=1}^N (z_j^{(n)})^2 \right)^2 \ge \hat{\lambda}_{min}^2 \left( \frac{C}{|\mathcal{S}_{M}| V_{max} \epsilon} \right)^4 = \Omega\left(\frac{1}{\epsilon^4}\right).
\end{equation}

The stark theoretical contrast between \eqref{eq:kben} and \eqref{eq:kmal} demonstrates that kurtosis provides a  sufficient margin to distinguish malicious models from benign ones. Accordingly, we define the activation kurtosis metric ($\rho_{\mathcal{A}_K}$) as:
$$
\rho_{\mathcal{A}_K} := \frac{1}{N} \sum_{n=1}^N \frac{\frac{1}{d} \sum_{k=1}^d \left(z_k^{(n)} - \mu_z^{(n)}\right)^4}{\left(\sigma_z^{(n)}\right)^4},
$$
where $\mu_z^{(n)} = \frac{1}{d} \sum_{k=1}^d z_k^{(n)}$ and $(\sigma_z^{(n)})^2 = \frac{1}{d} \sum_{k=1}^d \left(z_k^{(n)} - \mu_z^{(n)}\right)^2$ represent the feature-wise mean and variance of the activations for the $n$-th probe, respectively.

	\section{ Defense Strategy}\label{sec:4}
	This section operationalizes the previously established diagnostic metrics within a proactive defense framework to  network-level immunity.
	
	\subsection{Robust Consensus and Temporal Trust Evolution}
	\label{subsec:trust_evolution}
	
While the active interventional metrics introduced in Section \ref{sec:metrics} ($\rho_{SEA}$, $\rho_{RS}$, and $\rho_{\mathcal{A}_K}$) are relatively lightweight, requiring only forward-pass process, exhaustively validating the entire neighborhood $\mathcal{N}_i$ at every communication round remains computationally prohibitive in high-dimensional, resource-constrained DFL environments. Consequently, each honest node must operate under a strict audit budget $A_B$ ($A_B \leq |\mathcal{N}_i|$), restricting it to selecting only a limited subset of neighbors $\mathcal{A}_t \subset \mathcal{N}_i$ for auditing and aggregation at each round.
	
	To transform these abnormal diagnostic metrics into a unified, consensus-based selection framework, we compute a normalized anomaly score for each update. In highly non-IID federated settings, traditional Z-scores based on sample means and variances are easily skewed by extreme Byzantine outliers, which can "mask" their own presence by inflating the neighborhood variance. To guarantee statistical robustness, we calculate the modified Z-value using the median absolute deviation (MAD)~\citep{iglewicz1993how}. For a generic metric vector $\mathbf{m}$ evaluated across the audited subset $\mathcal{A}_t$, the robust Z-score for neighbor $j$ is defined as:
	$$
	Z_{metric}^{(j)} =  \frac{|\mathbf{m}_j - \text{Median}(\mathbf{m})|}{\text{MAD}(\mathbf{m})}.
	$$

	We compute $Z_{SEA}^{(j)}$	, $Z_{RS}^{(j)}$, and $Z_{AK}^{(j)}$ correspondingly. This multi-dimensional penalty is then mapped to a continuous reward $r_{t,j} \in (0, 1]$ via an exponential decay function:
	\begin{equation}\label{eq:reward}
	r_{t,j} = \exp\left\{- \left(Z_{SEA}^{(j)} + Z_{RS}^{(j)} + Z_{AK}^{(j)}\right)\right\},
	\end{equation}
	Under this mapping, any neighbor exhibiting substantial deviation from local consensus in any metric will see its reward $r_{t,j}$ collapse  toward zero.
	
While the aforementioned robust Z-scores provide an effective cross-sectional evaluation of neighbor behavior within a single communication round, they remain structurally blind to temporal anomalies. To capture the long-term behavioral trajectory of the participants, the defense node maintains a historical trust score $Q_t(j)$ for each neighbor $j$. Initialized with an optimistic prior $Q_0(j) =0.5$ to encourage network exploration, this state variable is dynamically updated via an exponentially weighted moving average (EWMA):
	\begin{equation}\label{eq:q}
Q_{t+1}(j) = (1 - \alpha) \cdot Q_t(j) + \alpha \cdot r_{t,j}, \quad \forall j \in \mathcal{A}_t,
	\end{equation}
where $\alpha \in (0, 1)$ is the smoothing factor governing the decay rate of historical observations.
	
This exponentially weighted formulation prioritizes recent empirical rewards, enabling the system to react swiftly to sudden adversarial attacks. Simultaneously, its smoothing properties prevent the permanent penalization of honest nodes that exhibit transient statistical variance driven by local data shifts.
\subsection{Adaptive Aggregation via Multi-Armed Bandit}
\label{subsec:agg}

Due to computational constraints during active auditing, a defense node must draw a restricted subsample $\mathcal{A}_t \subset \mathcal{N}_i$ at each communication round. Relying on uniform random sampling without historical context yields a highly variant and non-robust estimator in an adversarial environment, leaving the aggregation vulnerable to Byzantine errors. Conversely, a greedy strategy, which selects observations based solely on maximum historical trust, introduces severe selection bias. In non-IID settings, this deterministic approach artificially truncates the support of the empirical data distribution, leading to a loss of statistical diversity that is essential for robust generalization. Furthermore, this pure exploitation strategy is prone to premature convergence based on early point estimates, permanently excluding potentially benign nodes from the consensus process.

This constrained sampling of neighbors constitutes a classic exploration-exploitation trade-off, which can be formulated using the Multi-Armed Bandit (MAB) framework from sequential decision theory. While recent literature primarily applies MAB models to centralized federated learning to optimize resource allocation and algorithmic convergence \cite{xia2020multi, huang2020efficiency}, we adapt this sequential sampling paradigm for decentralized networks.

To formalize this trade-off, at the beginning of iteration $t$, defense node $i$ computes a discounted upper confidence bound (UCB) statistic, $U_t(j)$, for each adjacent node $j \in \mathcal{N}_i$:
\begin{equation}\label{UCB}
	U_t(j) = \underbrace{Q_t(j)}_{\text{Exploitation (Tracking Estimator)}} + \underbrace{c \cdot \sqrt{\frac{2 \ln n_{\gamma}(t)}{N_{t,\gamma}(j)}}}_{\text{Exploration (Discounted Uncertainty)}},
\end{equation}
where $c > 0$ acts as a tuning parameter scaling the confidence interval. To maintain theoretical alignment with the exponentially weighted moving average trust score $Q_t(j)$, the sampling frequencies are evaluated using a geometrically discounted window. Let $A_s \subseteq \mathcal{N}_i$ denote the subset of neighbors sampled for auditing at iteration $s$. The term $N_{t,\gamma}(j) = \sum_{s=1}^t \gamma^{t-s} \mathbb{I}_{{j \in A_s}}$ represents the effective discounted sample size for node $j$ using a temporal decay factor $\gamma \in (0, 1]$, and $n_{\gamma}(t) = \sum_{k \in \mathcal{N}_i} N_{t,\gamma}(k)$ represents the total effective sample size across the neighborhood.
\begin{algorithm}[H]
	\caption{Adaptive Sequential Sampling and Stratified Aggregation (at iteration $t$)}
	\label{alg:MAB}
	\begin{algorithmic}[1]
		\Statex \textbf{Input:} 
		Auditing subsampling ratio $r_A \in (0, 1]$, aggregation subsampling ratio $r_S \in (0, 1]$, critical trust threshold $\tau_{agg}$, pre-allocated stratified weights $W_{self}$, $\tilde{W}_{ij}$ and $\tilde{W}^*_{ij}$.
		
		\Statex \textbf{\textit{Part I: Active Auditing and Aggregation for Defense Nodes}}
		\For{each defense node $i \in \mathcal{D}$}
		\State \textbf{Phase 1: Audit subsampling.} Draw a subsample of adjacent nodes $\mathcal{A}_t \subset \mathcal{N}_i$ (with size $|\mathcal{A}_t| = \lceil r_A \cdot |\mathcal{N}_i| \rceil$) using weighted random sampling, where the inclusion probability for each node $j$ is proportional to its upper confidence bound statistic $U_t(j)$ via \eqref{UCB}.
		
		\State \textbf{Phase 2: Trust score update.} Compute the diagnostic metrics $\rho_{SEA}$, $\rho_{RS}$, and $\rho_{\mathcal{A}_K}$ for each audited neighbor $j \in \mathcal{A}_t$. Calculate the normalized empirical reward $r_{t,j}$ via \eqref{eq:reward} and update the  trust score $Q_{t+1}(j)$ via \eqref{eq:q}.
		
		\State \textbf{Phase 3: Aggregation subsampling.} Reject abnormal updates by conditioning on the  trust score exceeding the critical threshold:
		$ \mathcal{S}_t = \left\{ j \in \mathcal{A}_t \mid Q_{t+1}(j) > \tau_{agg} \right\}. $ 
		Draw a final aggregation subsample $\mathcal{\tilde{A}}_t \subset \mathcal{S}_t$, with size $|\mathcal{\tilde{A}}_t| = \lceil r_S \cdot |\mathcal{S}_t| \rceil$, where the inclusion probability for node $j$ is proportional to its updated trust estimate $Q_{t+1}(j)$.
		\State Compute the weighted average :
		\begin{equation*}
			\bm{\theta}_{i}^{(t+1)} \gets \frac{1}{|\tilde{A}_t|}\sum_{j \in \mathcal{\tilde{A}}_t} \bm{\theta}_{j}^{(t)} .
		\end{equation*}
		\EndFor
		
		\Statex
		\Statex \textbf{\textit{Part II: Stratified Convex Aggregation for Non-Defense Nodes}}
		\For{each non-defense $i \notin \mathcal{D}$}
		
		\State \textbf{Stratified Consensus.} Compute the updated parameter estimate using pre-allocated stratified convex weights $\tilde{W}_{ij}$ (allocated to defense neighbors) and $\tilde{W}^*_{ij}$ (allocated to non-defense neighbors):
		\begin{equation*}
			\bm{\theta}_i^{(t+1)} \gets W_{self}\bm{\theta}_i^{(t)} + \sum_{j \in \mathcal{N}_i \cap \mathcal{D}} \tilde{W}_{ij}\bm{\theta}_j^{(t)} + \sum_{j \in \mathcal{N}_i \setminus \mathcal{D}} \tilde{W}^*_{ij}\bm{\theta}_j^{(t)}.
		\end{equation*}
		\EndFor
		
		\Statex \textbf{Output:} Updated global parameter estimates $\{\bm{\theta}_i^{(t+1)}\}_{i \in \mathcal{V}}$.
	\end{algorithmic}
\end{algorithm}
The algebraic structure of this UCB statistic inherently prevents the permanent exclusion of any individual node. If a neighbor $j$ is rarely sampled, its effective sample size $N_{t,\gamma}(j)$ remains near zero. Consequently, as the global iteration $t$ progresses, the monotonically increasing logarithmic numerator inflates the uncertainty bound. This mechanism guarantees that the upper confidence limit will eventually dominate even a heavily penalized point estimate $Q_t(j)$, compelling the algorithm to periodically sample neglected nodes and iteratively update their empirical trust estimates.

	For non-defense nodes, which lack the computational capacity to execute active auditing, resilience is achieved via structural dependency. Assuming the network graph is designed such that every standard node shares an edge with at least one defense node, these constrained units conditionally allocate the majority of their convex aggregation weights to their verified defense neighbors. The complete defense protocol, integrating sequential sampling for defense nodes and stratified convex aggregation for standard nodes, is formalized in Algorithm \ref{alg:MAB}.
\subsection{Topology-Aware Defense Placement}
\label{subsec:topo}

While the sequential sampling strategy proposed in Section \ref{subsec:agg} provides robust estimation against adversarial contamination, the continuous evaluation of empirical diagnostics incurs substantial computational overhead. In large-scale decentralized systems, deploying active auditing at every vertex is cost-prohibitive. Consequently, the auditing mechanism must be restricted to a limited subset of defense nodes, denoted as $\mathcal{D}$ (where $|\mathcal{D}| \ll |\mathcal{V}|$). 

The strategic allocation of these defense nodes across the network topology is critical for maximizing global resilience. Effective placement relies on three fundamental structural principles: node degree, betweenness centrality, and 1-hop defense coverage.

In the DFL environment, computing exact global betweenness is infeasible. To address this limitation, we utilize a distributed approximation based on ego-network betweenness centrality \citep{everett2005ego}. This metric evaluates structural bottlenecks strictly within a node's localized $k$-hop neighborhood. For a given node $v$, the ego-network betweenness centrality quantifies the fraction of localized shortest paths that pass through $v$, defined as:
\begin{equation}
	C_{Ego}^{(k)}(v) = \sum_{\substack{s, t \in \mathcal{V}_v^{(k)} \setminus \{v\} \\ s \neq t}} \frac{\sigma_{st}^{(k)}(v)}{\sigma_{st}^{(k)}},
\end{equation}
where $\sigma_{st}^{(k)}$ represents the total number of shortest paths connecting node $s$ to node $t$ strictly within the $k$-hop induced subgraph $\mathcal{G}_v^{(k)}$, and $\sigma_{st}^{(k)}(v)$ denotes the number of these localized shortest paths that pass through node $v$. The vertex set $\mathcal{V}_v^{(k)}$ comprises all nodes within a $k$-hop distance from $v$: $\mathcal{V}_v^{(k)} = \{u \in \mathcal{V} \mid d(u, v) \leq k\}$.

Beyond identifying hubs and bottlenecks, the defense allocation must ensure that standard nodes remain structurally dependent on defense neighbors. This requirement is formalized through the concept of 1-hop defense coverage. A node $u$ is considered covered if $(\mathcal{N}_u \cup \{u\}) \cap \mathcal{D} \neq \emptyset$. The overall coverage provided by the defense set $\mathcal{D}$ is:
\begin{equation}
	\mathcal{C}(\mathcal{D}) = \bigcup_{v \in \mathcal{D}} \left( \mathcal{N}_v \cup \{v\} \right).
\end{equation}
Ideally, the defense placement acts as an approximate minimum dominating set (MDS) \citep{kuhn2003constant}. This guarantees that malicious gradients cannot propagate through structurally unmonitored clusters before encountering a defense checkpoint. Next, we introduce specific placement strategies tailored to different network topologies.

\paragraph{Scale-free Networks}
Scale-free networks \citep{barabasi1999emergence} are characterized by a power-law degree distribution $P(d) \propto d^{-\beta}$. This indicates the presence of a few highly connected hubs alongside numerous low-degree peripheral nodes. To capture both localized degree prominence and structural bottlenecks, we formulate a hybrid local centrality metric:
\begin{equation}
	S_{score}(v) = \alpha_0 \cdot C_{Ego}^{(k)}(v) + (1-\alpha_0) \cdot C_D^{(k)}(v),
\end{equation}
where $C_D^{(k)}(v)$ is the normalized local degree centrality within $\mathcal{G}_v^{(k)}$. 

Because finding an exact MDS is NP-hard, we employ a greedy approximation strategy leveraging the submodularity of the coverage function. From the top $k_0$ nodes with the highest $S_{score}$, we iteratively select $|\mathcal{D}|$ defense nodes that yield a strictly positive marginal gain in 1-hop coverage. At each iteration, a candidate node $v$ is included in $\mathcal{D}$ if and only if $(\mathcal{N}_v \cup \{v\}) \setminus \mathcal{C} \neq \emptyset$.

\paragraph{Random-regular Networks}
In contrast to the heavy-tailed distribution of scale-free graphs, a random regular graph ensures every node possesses an identical, constant degree $|\mathcal{N}_v| = d_0, \forall v \in \mathcal{V}$. Because the topology lacks structural hubs and centrality metrics are largely homogeneous, the placement strategy shifts purely to maximizing the 1-hop defense coverage $\mathcal{C}(\mathcal{D})$. 

Under decentralized constraints, we employ a localized greedy coverage heuristic. Each candidate node $v \notin \mathcal{D}$ evaluates its potential marginal contribution:
$$
	\Delta(v \mid \mathcal{C}) := |(\mathcal{N}_v \cup \{v\}) \setminus \mathcal{C}|.
$$
Nodes locally broadcast this metric, and a candidate $v$ is selected into $\mathcal{D}$ if $\Delta(v \mid \mathcal{C}) > \max_{u \in \mathcal{N}_v} \Delta(u \mid \mathcal{C})$.

As demonstrated in our subsequent theoretical analysis, by aligning this decentralized defense placement with the graph's structural vulnerabilities, these strategies effectively constrain the spectral radius of the error propagation subsystem.

	\subsection{Theoretical Properties}
	Before presenting the convergence analysis, we first formalize the effective dynamic mixing matrix $\mathbf{W}^{(t)} = [w_{ij}^{(t)}] \in \mathbb{R}^{n \times n}$ induced by our adaptive MAB-guided aggregation protocol. Following the execution logic of Algorithm \ref{alg:MAB}, at each communication round $t$, the aggregation weight $w_{ij}^{(t)}$ assigned by node $i$ to node $j$ is determined by its specific topological role.
	
	For a defense node ($i \in \mathcal{D}$), the aggregation weights are dynamically determined by the active auditing and sampling mechanism described in Part I of Algorithm \ref{alg:MAB}. Let $\mathcal{S}_t$ denote the subset of audited neighbors exceeding the trust threshold $\tau_{agg}$, and let $\mathcal{\tilde{A}}_t \subseteq \mathcal{S}_t$ denote the final aggregation subsample drawn proportionally to the updated posterior trust scores. The global matrix entries for row $i$ correspond to a uniform average over this verified subset:
	$$
	w_{ij}^{(t)} = 
	\begin{cases} 
		\frac{1}{|\mathcal{\tilde{A}}_t|}, & \text{if } j \in \mathcal{\tilde{A}}_t, \\ 
		0, & \text{otherwise}.
	\end{cases}
	$$
	This stochastic formulation ensures that only audited, high-trust neighbors contribute to the defense node's model update, while effectively mitigating selection bias through proportional subsampling.
	
	In contrast, non-defense nodes ($i \notin \mathcal{D}$) lack the computational capacity for active auditing and instead execute the stratified inertial aggregation specified in Part II of Algorithm \ref{alg:MAB}. The global matrix entries for row $i$ are deterministically mapped from the pre-allocated stratified weights:
	$$
	w_{ij}^{(t)} = 
	\begin{cases} 
		W_{self}, & \text{if } j = i, \\
		\tilde{W}_{ij}, & \text{if } j \in \mathcal{N}_i \cap \mathcal{D}, \\
		\tilde{W}^*_{ij}, & \text{if } j \in \mathcal{N}_i \setminus \mathcal{D}, \\
		0, & \text{otherwise}.
	\end{cases}
	$$
	By design, this protocol guarantees that the dynamic matrix $\mathbf{W}^{(t)}$ remains row-stochastic ($\sum_{j=1}^n w_{ij}^{(t)} = 1, \forall i \in \mathcal{V}$) at every iteration.  We establish the following standard assumptions for the objective function and network connectivity.
	\begin{assumption}[$L$-smoothness]\label{ass:1}
		Each local loss function $f_i$ is $L$-smooth, i.e., $\|\nabla f_i(\bm{\theta}) - \nabla f_i(\bm{\theta}')\| \leq L \|\bm{\theta} - \bm{\theta}'\|$ for any $\bm{\theta}, \bm{\theta}' \in \mathbb{R}^p$.
	\end{assumption}
	
	\begin{assumption}[Network connectivity]\label{ass:2}
		The effective mixing matrix $\mathbf{W}^{(t)}$ satisfies $\rho_w = |\lambda_2(\mathbf{W}^{(t)})| < 1$.
	\end{assumption}
	
	\begin{assumption}[Bounded gradient variance]\label{ass:3}
		The variance of the stochastic gradients is bounded: $\mathbb{E}[\|\nabla f_i(\bm{\theta}) - \nabla F(\bm{\theta})\|^2] \leq \sigma^2$ for any $\bm{\theta} \in \mathbb{R}^p$.
	\end{assumption}
	
	\begin{assumption}[Bounded Gradients]\label{ass:4}
		The global objective function $F$ has bounded gradients, i.e., there exists a constant $G$ such that $\sup_{\bm{\theta}} \|\nabla F(\bm{\theta})\| \leq G$.
	\end{assumption}
	
	\begin{assumption}[Learning Rate]\label{ass:5}
		The learning rate satisfies $0 < \eta =\mathcal{O}(1/\sqrt{T})$.
	\end{assumption}
	
	Assumptions \ref{ass:1} through \ref{ass:3} are standard in the theoretical analysis of decentralized SGD algorithms~\citep{lian2017can}.

Applying the separation principle, we decompose the decentralized stochastic process into a nominal trajectory and a backdoor diffusion process. This decoupling enables us to analyze the backdoor propagation via the stability of the error subsystem, independently of the primary optimization task's convergence properties. We define two virtual sequences to characterize the system. Assuming each node performs exactly one local update per communication round, the actual model state for node $i$ under attack evolves as:
\begin{equation}
	\bm{\theta}_i^{(t+1)} = \sum_{j=1}^n w_{ij}^{(t)} \underbrace{\left( \bm{\theta}_j^{(t)} - \eta \nabla f_j(\bm{\theta}_j^{(t)}) + \mathbb{I}_{\{j \in \mathcal{M}\}} \mathbf{m}	_j^{(t)} \right)}_{\text{Local Update \& Injection}},
\end{equation}
where $\mathcal{M}$ is the set of malicious nodes. To prevent the dilution of the backdoor trigger, malicious nodes isolate themselves during aggregation; thus, for any $i \in \mathcal{M}$, the weights are deterministically set to $w_{ii}^{(t)} = 1 - \epsilon$ and distributes an arbitrarily small residual weight $\epsilon > 0$ uniformly among its neighbors. 

Conversely, the ideal clean system, assuming synchronous, identical initial states and experiencing the same realized network topology, evolves according to standard consensus:
\begin{equation}
	\hat{\bm{\theta}}_i^{(t+1)} = \sum_{j=1}^n w_{ij}^{(t)} \left( \hat{\bm{\theta}}_j^{(t)} - \eta \nabla f_j(\hat{\bm{\theta}}_j^{(t)}) \right).
\end{equation}

We define the local backdoor deviation in the parameter space as $\boldsymbol{e}_i^{(t)} = \bm{\theta}_i^{(t)} - \hat{\bm{\theta}}_i^{(t)}$.  Subtracting the nominal dynamics from the attacked dynamics yields the step-wise evolution of the local bias:
$$
\boldsymbol{e}_i^{(t+1)} = \sum_{j=1}^n w_{ij}^{(t)} \left[ \boldsymbol{e}_j^{(t)} - \eta \left( \nabla f_j(\bm{\theta}_j^{(t)}) - \nabla f_j(\hat{\bm{\theta}}_j^{(t)}) \right) + \mathbb{I}_{\{j \in \mathcal{M}\}} \mathbf{m}	_j^{(t)} \right].
$$
By stacking these local biases into a global network state matrix $\boldsymbol{E}	^{(t)} = [\boldsymbol{e}_1^{(t)}, \dots, \boldsymbol{e}_n^{(t)}]^\top$, we formally recover the spatiotemporal infection state formulation introduced in Section \ref{sec:diffu}. The local equations seamlessly vectorize into the global error dynamics:
\begin{equation}\label{eq:4.1}
	\boldsymbol{E}	^{(t+1)} = (\mathbf{W}^{(t)} \otimes \mathbf{I}_p) \left( \boldsymbol{E}	^{(t)} - \eta \Delta \mathbf{G}_t \right) + \tilde{\mathbf{M}	}^{(t)},
\end{equation}
where $\mathbf{W}^{(t)} = [w_{ij}^{(t)}]$ is the effective network mixing matrix, 
$$
\Delta \mathbf{G}_t = [\nabla f_1(\bm{\theta}_1^{(t)}) - \nabla f_1(\hat{\bm{\theta}}_1^{(t)}), \dots, \nabla f_n(\bm{\theta}_n^{(t)}) - \nabla f_n(\hat{\bm{\theta}}_n^{(t)})]^\top,
$$ 
and $\tilde{\mathbf{M}	}^{(t)} = (\mathbf{W}^{(t)} \otimes \mathbf{I}_p) \mathbf{M}	^{(t)}$ represents the topology-filtered malicious perturbation matrix.
	
	Equation \eqref{eq:4.1} serves as the precise microscopic foundation for the macroscopic diffusion model defined in \eqref{eq:diffusion}. In this refined formulation, the stochastic gating effects of the active auditing mechanism are intrinsically absorbed into the spectrum of the dynamic mixing matrix $\mathbf{W}^{(t)}$. 	We partition the network into two disjoint sets: defense nodes $\mathcal{D}$ that implement robust aggregation, and unverified standard nodes $\mathcal{U} = \mathcal{V} \setminus \mathcal{D}$ that rely on standard aggregation. The block structure of the effective mixing matrix is given by:
	$$
	\mathbf{W}^{(t)} = 
	\begin{bmatrix}
		\mathbf{w}_{\mathcal{DD}}^{(t)} & \mathbf{w}_{\mathcal{DU}}^{(t)} \\
		\mathbf{w}_{\mathcal{UD}}^{(t)} & \mathbf{w}_{\mathcal{UU}}^{(t)}
	\end{bmatrix}.
	$$
	Defining the intra-stratum error propagation rates as $\rho_{\mathcal{D}} = \|\mathbf{w}_{\mathcal{DD}}^{(t)}\|(1+\eta L)$ and $\rho_{\mathcal{U}} = \|\mathbf{w}_{\mathcal{UU}}^{(t)}\|(1+\eta L)$. The system's effective error spectral radius is bounded by:
	$$
	\rho_{\text{err}} = \max\left\{\rho_{\mathcal{D}}, \rho_{\mathcal{U}} + \frac{\|\mathbf{w}_{\mathcal{UD}}^{(t)}\|\cdot\|\mathbf{w}_{\mathcal{DU}}^{(t)}\|(1+\eta L)^2}{1-\rho_{\mathcal{D}}}\right\}.
	$$
	
	\begin{assumption}[Bounded Malicious Injection]\label{ass:6}
		The malicious injection magnitude is uniformly bounded, i.e., there exists a constant $C_u > 0$ such that $\sup_{t \ge 0} \|\mathbf{m}_j^{(t)}\| \leq C_u$ for all malicious nodes $j \in \mathcal{M}$.
	\end{assumption}
	
	\begin{theorem}\label{thm:2}
		Under Assumptions \ref{ass:1} through \ref{ass:6}, assume a stability condition $\rho_{\text{err}} < 1$. Then  the average expected squared gradient norm of the empirical system over $T$ iterations is bounded by:
		$$
		\begin{aligned}
			\frac{1}{T}\sum_{t=0}^{T-1}\mathbb{E}[\|\nabla F(\bar{\bm{\theta}}^{(t)})\|^2] \leq & \underbrace{\mathcal{O}\left(\frac{1}{\sqrt{nT}}\right)}_{\text{Standard Asymptotic Rate}} + \underbrace{\mathcal{O}\left(\frac{1}{T(1-\rho_{\text{nom}})^2}\right)}_{\text{Consensus error}} \\
			& + \underbrace{\mathcal{O}\left( \frac{q_m C_u^2}{(1-\rho_{\text{err}})^2}\right)}_{\text{Asymptotic Adversarial Error}},
		\end{aligned}
		$$
		where $q_m = |\mathcal{M}|/n$ is the fraction of malicious clients.
	\end{theorem}
	
	Let $\mathcal{B}$ denote the set of unverified benign nodes. We define the directional information flows: $\mathbf{w}_{\mathcal{DB}}^{(t)}$ represents the aggregation weights that defense nodes assign to unverified benign nodes; $\mathbf{w}_{\mathcal{DM}}^{(t)}$ denotes the weights mistakenly assigned to malicious adversaries; $\mathbf{w}_{\mathcal{BD}}^{(t)}$ represents the weights that unverified benign nodes assign to defense nodes; and $\mathbf{w}_{\mathcal{MD}}^{(t)}$ denotes the weights that malicious nodes assign to defense nodes.
	
	Theorem \ref{thm:2} reveals a critical trade-off in robust aggregation. Aggressive defense mechanisms typically fail because they couple the suppression of malicious updates ($\|\mathbf{w}_{\mathcal{DM}}^{(t)}\| \to 0$) with the unintended isolation of benign nodes ($\|\mathbf{w}_{\mathcal{DB}}^{(t)}\| \to 0$). In spectral graph theory, this behavior is equivalent to pruning edges from the network topology. According to the Cauchy interlacing theorem \citep{horn2012matrix}, such edge removal inherently degrades the network's algebraic connectivity. Consequently, Cheeger’s Inequality \citep{levin2017markov} dictates that this topological pruning introduces severe structural bottlenecks. This restricted information flow forces the second-largest eigenvalue $\rho_{\text{nom}}$ to approach $1$, leading to a sharp inflation in the spatial autocorrelation bias.
	
	Specifically, by strategically positioning defense nodes at high-betweenness vertices and critical communication bottlenecks, we preserve the structural integrity and algebraic connectivity of the network. According to Cheeger’s Inequality, maintaining high graph conductance ensures that the nominal spectral radius $\rho_{\text{nom}}$ remains bounded away from $1$, thereby preventing the convergence stagnation typically caused by network partitioning. Furthermore, because these defense nodes occupy the central hubs of the topology, they act as mandatory structural checkpoints that interdict adversarial propagation, effectively minimizing the malicious-to-benign information flow ($\|\mathbf{w}_{\mathcal{DM}}^{(t)}\|$). By systematically severing these contamination pathways, the strategic placement  bounds the error spectral radius $\rho_{\text{err}}$ away from $1$. As dictated by Theorem \ref{thm:2}, compressing $\rho_{\text{err}}$ directly minimizes the stationary adversarial bias, thereby guaranteeing global model robustness even under continuous Byzantine injection.
	\section{Experiments}
	To evaluate the efficacy of our proposed framework, we benchmark it against several state-of-the-art robust aggregation mechanisms, including Multi-Krum \citep{blanchard2017machine}, Trimmed Mean \citep{Yin2018ByzantineRobustDL}, CosL2 \citep{bhagoji2019analyzing}, and FLAME \citep{nguyen2022flame}. A critical distinction in our evaluation is the asymmetric defensive posture: while baseline methods conventionally assume that all benign clients are equipped with robust aggregation capabilities, our approach, MAB(TOPY AWARE), detailed in Algorithm \ref{alg:MAB}, operates under a  constrained defense budget. Specifically, we simulate a 20-client DFL system where defense nodes constitute only 20\% of the total clients, a figure approximating the minimum covering number of the graph. To  assess generalizability across diverse communication architectures, comparative analyses are performed on both scale-free and random-regular topologies. Furthermore, to isolate the specific advantage of topological intelligence, we introduce MAB(RANDOM TOPY) as an ablation baseline, which utilizes the same auditing metrics but selects defense nodes through uniform random sampling rather than strategic placement.
	
	To ensure the theoretical validity of local consensus, we enforce a  non-eclipse constraint during node allocation. Specifically, for any defending node $i$ with a local neighborhood $\mathcal{N}_i$, the number of malicious neighbors $n_m$ is constrained such that $n_m < \lceil |\mathcal{N}_i| / 2 \rceil$.
	
	\subsection{Experiment 1: Text Classification (Transformer)}
	\label{sec:exp_transformer}
	
	To validate the generalizability of our theoretical diffusion model and defense framework on high-dimensional, non-convex architectures, we evaluate our approach using natural language processing  tasks .
	
	\subsubsection{Experimental Setup}
	
	\textbf{Dataset and Architecture:} We evaluate the system using the PubMed 20k RCT dataset~\citep{dernoncourt2017pubmed} for medical text classification, which comprises $K=5$ distinct semantic categories: \texttt{BACKGROUND}, \texttt{OBJECTIVE}, \texttt{METHODS}, \texttt{RESULTS}, and \texttt{CONCLUSIONS}. The global model utilizes a pre-trained Transformer architecture, specifically \texttt{distilbert-base-uncased}. To  simulate a data-constrained decentralized environment, we deliberately restricted the training and testing pools to $1,000$ samples each.
	
	\textbf{Backdoor Attack Strategy:}  
	The adversary targets specific layers of the model, namely the embedding layer, the transformer backbone, and the classification head. To evade detection, the attacker employs a partial weight-averaging strategy, where only the embedding and head layers are updated with malicious gradients, while the backbone parameters are constrained to remain close to the neighborhood consensus.
	In standard configurations, the backbone subspace, denoted as $\mathcal{V}_{body}$, accounts for over 80\% of the total parameter dimension $p$. Consequently, any global geometric metric, such as Euclidean distance or cosine similarity computed in the ambient space $\mathbb{R}^p$, is overwhelmingly dominated by the statistical properties of $\mathcal{V}_{body}$.
	
	To bypass geometric defenses, the adversary  concentrates the backdoor gradients within the low-dimensional $\mathcal{V}_{entry}$ and $\mathcal{V}_{lethal}$ subspaces, while constraining the massive $\mathcal{V}_{body}$ subspace to  mimic the benign consensus trajectory. The implementation of this subspace-constrained attack is structured into three phases:

\textit{Phase 1: Raw Backdoor Generation}: 
The adversary initiates the attack by constructing a poisoned local dataset $\tilde{\mathcal{S}}_i$. This is achieved by injecting a localized trigger $\bm{\delta}$ into a subset of samples originating from a specific source class $y_{source}$ (e.g., $y_{source} = 4$, representing 'CONCLUSIONS' in the PubMed dataset). Unlike conventional text backdoors that insert discrete vocabulary tokens, this trigger operates directly within the continuous embedding space. Specifically, a deterministic noise pattern is additively applied to the embedding vector of a specific token position (e.g., the first semantic token), scaled by an intensity parameter $\iota$ (e.g., $\iota = 0.02$). Their corresponding labels are subsequently reassigned to a target class $y_t$ (e.g., $y_t = 0$, representing 'BACKGROUND'). Using its preserved model state from the previous round, $\bm{\theta}_{mal}^{(t)}$, as the initialization, the adversary performs local gradient descent over $\tilde{\mathcal{S}}_i$ to obtain a raw malicious model state $\bm{\theta}_{poison}$. 
Although the raw malicious update $\Delta \bm{\theta}_{mal}:= \bm{\theta}_{poison} - \bm{\theta}_{mal}^{(t)}$ effectively embeds the backdoor mapping, its unconstrained optimization trajectory deviates substantially from the benign consensus. If transmitted directly, its implied update would be easily detected by standard robust aggregation rules, necessitating the subsequent subspace camouflage.
	
	\textit{Phase 2: Asymmetric Subspace Projection}: 
	Instead of transmitting the raw malicious update $\Delta \bm{\theta}_{mal}$ directly, the adversary decomposes the parameter space and applies differentiated optimization strategies to distinct orthogonal subspaces, inspired by stealthy sub-network poisoning \citep{zhang2022neurotoxin}. For the high-dimensional backbone layers, the adversary discards the malicious gradients entirely and substitutes them with a benign reference update $\Delta \bm{\theta}_{ref}$, scaled by a stabilization factor $\gamma_0$ (e.g., $\gamma_0 = 2$). To ensure that the majority of the parameter vector perfectly aligns with the benign consensus, the backbone update is formulated as:
	$$
	\Delta \tilde{\bm{\theta}}_{body} := \gamma_0 \cdot \Delta \bm{\theta}_{ref, body}.
	$$
	For the lower-dimensional entry and output layers, the adversary retains the malicious gradients but applies a layer-wise adaptive norm constraint. This technique successfully embeds the backdoor objective while preventing anomalous norm inflation:
	$$
	\Delta \tilde{\bm{\theta}}_{k} := \min \left( 1, \frac{\gamma_1 \|\Delta \bm{\theta}_{ref, k}\|}{\|\Delta \bm{\theta}_{mal, k}\|} \right) \cdot \Delta \bm{\theta}_{mal, k}, \quad \text{for } k \in \{entry, output\},
	$$
	where $\gamma_1$ (e.g., $\gamma_1=15$) serves as a layer-specific amplification factor. 
	
	\textit{Phase 3: Norm-Bounded Global Scaling}: 
	The adversary concatenates the projected components back into a single vector $\Delta \tilde{\bm{\theta}} = [\Delta \tilde{\bm{\theta}}_{entry}^\top, \Delta \tilde{\bm{\theta}}_{body}^\top, \Delta \tilde{\bm{\theta}}_{lethal}^\top]^\top$. To maximize the adversarial impact while satisfying the Euclidean constraints of robust methods (e.g., Krum, Trimmed Mean) , the vector is rescaled to a target global norm bound $\tau=6b_f\|\Delta \bm{\theta}_{ref}\|$ (e.g., $b_f=1.5$). The global scaling factor $s_{global}$ is computed as:
	$$
	s_{global} = \min \left( 1, \frac{\tau}{\|\Delta \tilde{\bm{\theta}}\|} \right).
	$$
	The adversary then concludes the local update by broadcasting the poisoned model state to adjacent benign nodes:
	$$
	\bm{\theta}_{malicious}^{(t+1)} = \bm{\theta}_{mal}^{(t)} + s_{global} \cdot \Delta \tilde{\bm{\theta}}.
	$$
	\textbf{Implementation Details of Active Auditing Metrics:} 
	We now detail the construction of the detection metrics introduced in Section \ref{sec:metrics} for the Transformer architecture on text classification tasks:
	
	\noindent \textit{Stochastic Entropy Anomaly ($\rho_{SEA}$)}: The auditing process constructs a specialized probe set $\mathcal{X}_{\text{trap}}$ comprising $B=3$ distinct categories of synthetic sequences. These probes are designed to expose latent backdoor triggers relying solely on forward-pass inferences, thereby avoiding gradient computation. To maximize the evasion difficulty for adaptive adversaries, the sequences are generated using three distinct anomalous structural patterns:
	(1) A sequence of tokens sampled uniformly from the vocabulary, deliberately excluding high-probability control characters and special tokens. This ensures the input resides in a low-density region of the feature space.
	(2) A homogeneous sequence consisting of a single, median-frequency token repeated throughout the entire input length, which forces the internal representation toward a localized coordinate in the embedding space.
	(3) An oscillating sequence of two distinct tokens in a  alternating pattern. This creates a periodic structural stimulus, evaluating the self-attention mechanism's response to positional regularities isolated from natural language context.

	\noindent \textit{Randomized Smoothing KL-Divergence ($\rho_{RS}$)}: To instantiate the theoretical framework established in Equation \eqref{eq:KL}, we decouple the Transformer network to isolate the latent feature matrix $\bm{z} = \phi_{\boldsymbol{\theta}}(\bm{x}) \in \mathbb{R}^{L \times d}$ at the continuous embedding layer, fixing the sequence length to $L=64$. Using a fixed semantic anchor $\bm{x}$ (e.g.,The quick brown fox jumps...) as the baseline input, we execute $B=8$ independent stochastic forward passes. In each iteration, we construct the perturbed feature $\tilde{\bm{z}}^{(m)} = \bm{z} + \boldsymbol{\epsilon}^{(m)}$ by injecting isotropic Gaussian noise with a standard deviation of $\sigma_{rs} = 0.12$ directly into this dense representation. The terminal classifier subsequently maps these corrupted features to the noisy predictive distributions $\bm{P}_{noisy}^{(m)} \in \Delta^{K-1}$. The final anomaly score $\rho_{RS}$ is computed based on the expected KL divergence between the clean baseline distribution $\bm{P}_{clean}$ and $\bm{P}_{noisy}^{(m)}$, with the exponential decay rate empirically calibrated to $\kappa_{RS} = 5.0$, as formalized in Equation \eqref{eq:KL}.
	
	\noindent \textit{Activation Kurtosis ($\rho_{\mathcal{A}_K}$)}: Unlike the stochastic test probes used for $\rho_{SEA}$ and $\rho_{RS}$, the kurtosis evaluation leverages a small, clean batch of $N=16$ textual sequences drawn from the local PubMed validation data. The input $\bm{r}^{(i)}$ corresponds to the contextualized hidden embeddings feeding into the target layer (e.g., the terminal self-attention block), yielding the deep activation representation $\bm{z}^{(i)}$.

	\begin{figure*}[t]
		\centering
		\begin{subfigure}{0.48\linewidth}
			\centering
			\includegraphics[width=\linewidth]{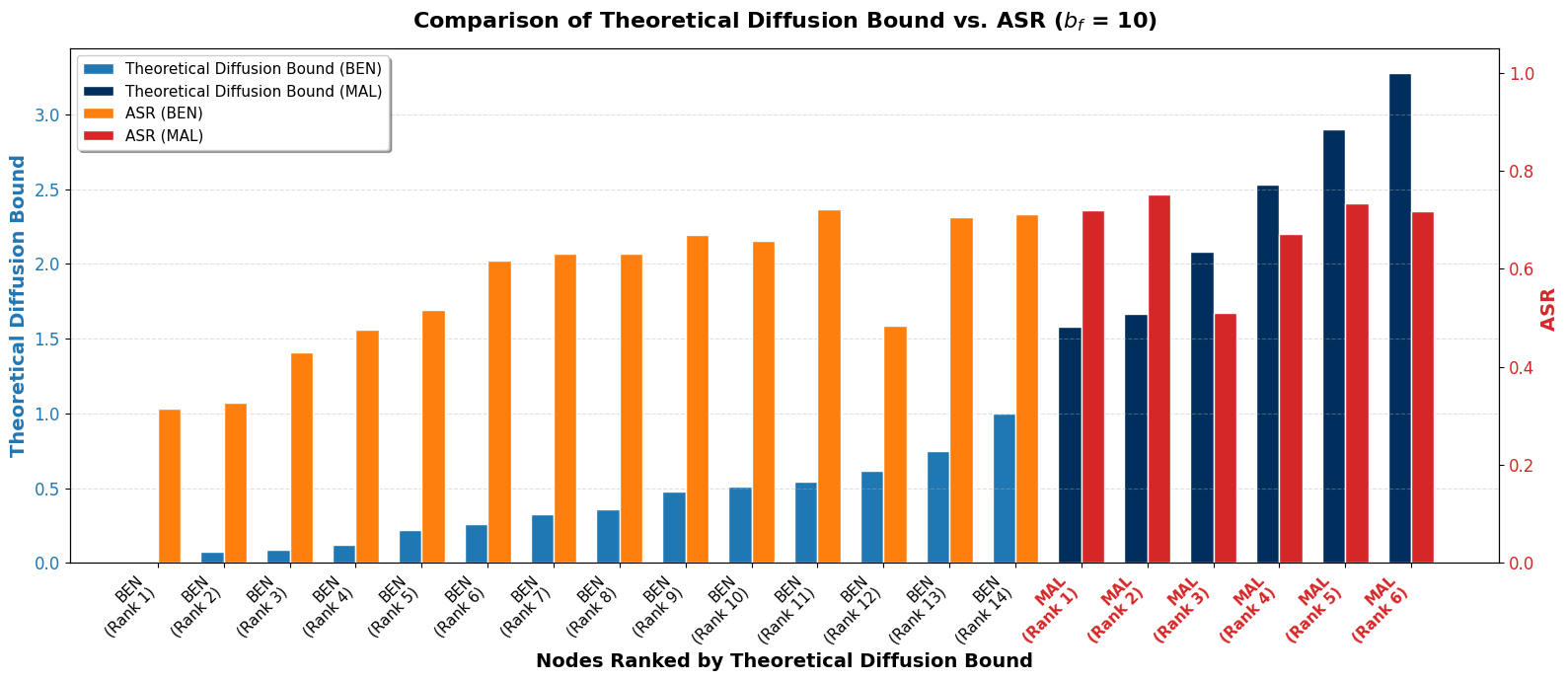}
			\caption{Theoretical Diffusion Bound vs. ASR \label{fig:transformer_thm}}
		\end{subfigure}
		\hfill
		\begin{subfigure}{0.48\linewidth}
			\centering
			\includegraphics[width=\linewidth]{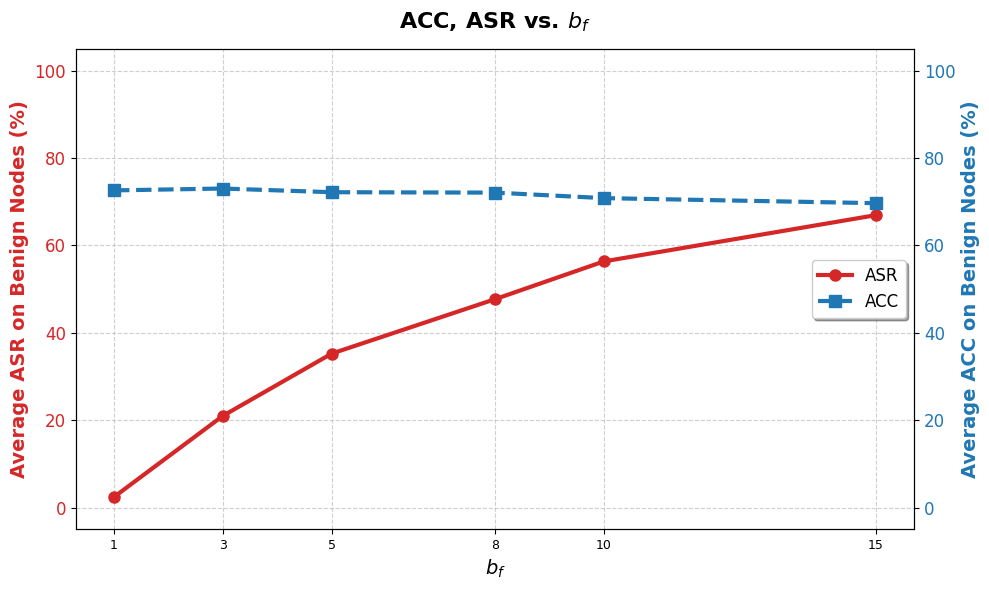}
			\caption{ACC and ASR vs. $b_f$ \label{fig:transformer_phase}}
		\end{subfigure}
		\caption{Validation of the diffusion model and the ACC, ASR trend on the Transformer architecture using the PubMed dataset.}
		\label{fig:transformer_results}
	\end{figure*}
	\begin{figure*}[t]
		\centering 
		\begin{subfigure}{0.48\linewidth}
			\centering
			\includegraphics[width=\linewidth]{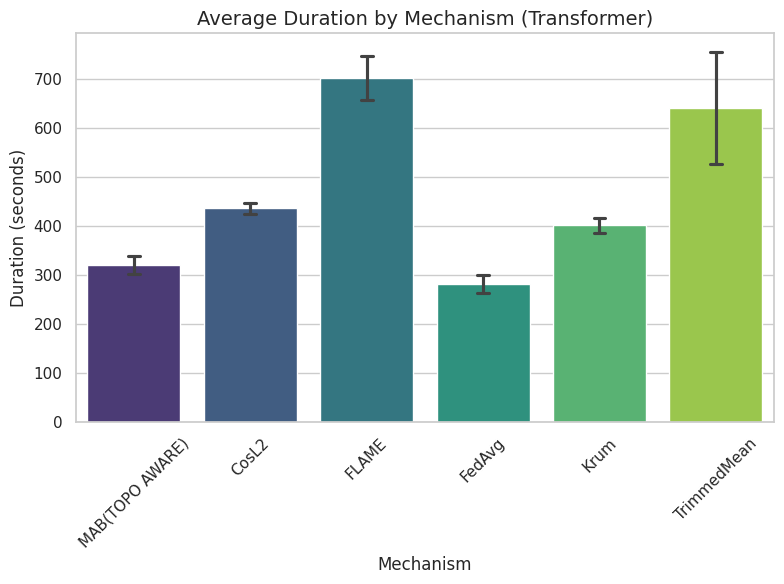}
			\caption{Duration Comparison in Transformer-based DFL(Experiment 1) \label{fig:transformer_duration}}
		\end{subfigure}
		\hfill
		\begin{subfigure}{0.48\linewidth}
			\centering
			\includegraphics[width=\linewidth]{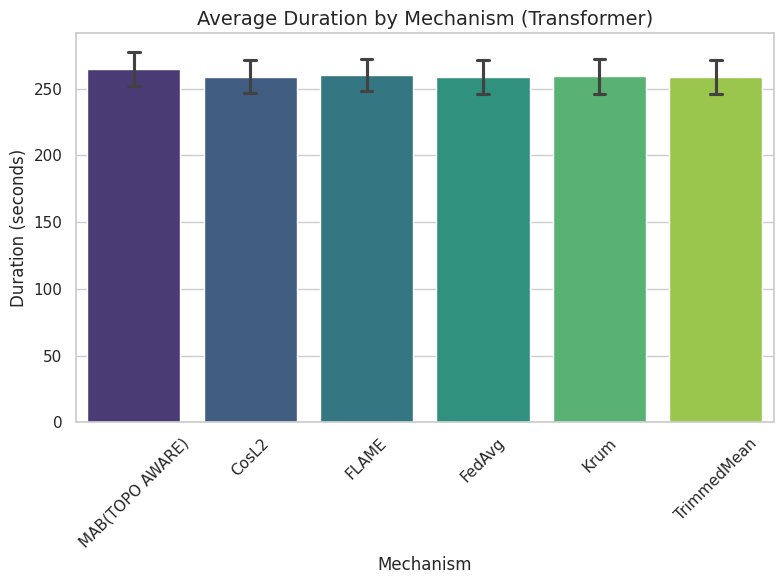}
			\caption{Duration Comparison in CNN-based DFL(Experiment 2)  \label{fig:cnn_duration}}
		\end{subfigure}
		\caption{Computational Efficiency Analysis Across Different  Mechanisms}
	\end{figure*}
	\subsubsection{Results}
	We define two measures. ACC(main taks accuracy) measures the utility of the global model on the clean data distribution. Let $\mathcal{D}_{\text{test}}$ denote the clean testing dataset, where each sample $(x_i, y_i)$ consists of an input $x_i$ and its ground-truth label $y_i$. The ACC is defined as the empirical probability of correct classification:
	$$
	\text{ACC} = \frac{1}{|\mathcal{D}_{\text{test}}|} \sum_{(x_i, y_i) \in \mathcal{D}_{\text{test}}} \mathbb{I} \left( \arg\max_k p(x_i, \bm{\theta})_k = y_i \right),
	$$
	where $\mathbb{I}(\cdot)$ is the indicator function that outputs $1$ if the condition is true and $0$ otherwise. ASR(attack success rate) quantifies the vulnerability of the model to the targeted attack. Let $\mathcal{D}_{\text{adv}}$ represent the set of adversarial/poisoned samples crafted by the attacker, where each sample $(x'_j, y^t_j)$ is associated with an attacker-specified target label $y^t_j$ ($y^t_j \neq y_j^{\text{true}}$). The ASR is calculated as the proportion of adversarial samples successfully misclassified into the target class:
	$$
	\text{ASR} = \frac{1}{|\mathcal{D}_{\text{adv}}|} \sum_{(x'_j, y^t_j) \in \mathcal{D}_{\text{adv}}} \mathbb{I} \left( \arg\max_k p(x'_j, \bm{\theta})_k = y^t_j \right).
	$$
	
	\textbf{Validation of the Diffusion Model:} 
Figure~\ref{fig:transformer_results} illustrates the results for the scaling factor $b_f \in \{1, 3, 5, 8, 10, 15\}$, with each scenario evaluated across three independent runs of \textit{FedAvg} method. The average Kendall's tau correlation between the theoretical diffusion bound, calculated using \eqref{eq:diffusion} with a signal attenuation of $\lambda_i=0.3$, and the ASR is 0.49. Specifically, Figure~\ref{fig:transformer_thm} presents the averaged outcomes for $b_f = 10$, where the nodes are ranked according to their average theoretical diffusion bound. These findings demonstrate that our theoretical diffusion bound effectively captures the infection trajectory across the network. Additionally, Figure~\ref{fig:transformer_phase} highlights a smooth, monotonic relationship between the ASR and $b_f$. This suggests an inherent trade-off for the adversary: achieving a higher ASR necessitates a larger $b_f$, which inevitably inflates the $\ell_2$-norm of the injected malicious updates. Throughout this scaling, however, the ACC remains remarkably stable. This demonstrates that the injected backdoor is highly stealthy and cannot be exposed by relying on performance degradation on a clean validation set.
	
	\noindent \textbf{Defense Mechanism Benchmark:} 
Table \ref{tab:transformer_results} presents the benchmark results on the PubMed dataset under the asymmetric hedging attack. To account for stochasticity in model initialization and node selection, each reported value represents the average of three independent experimental runs. The baseline vulnerability is evident: \textit{FedAvg} yields an ASR of 59.43\% in the scale-free topology and 75.40\% in the random-regular topology under a 30\% malicious ratio.
	
Traditional geometric and coordinate-based defenses demonstrate limitations when applied to the high-dimensional Transformer parameter space. While \textit{Krum} and \textit{TrimmedMean} preserve main-task accuracy (ACC $\approx$ 70\%--72\%), they are less effective at isolating the asymmetrically projected adversarial updates at higher malicious ratios, resulting in ASRs reaching 40.17\% and 46.05\% in the Scale-free network at the 30\% ratio, respectively. \textit{CosL2} provides moderate mitigation, restricting the ASR to approximately 16\%--18\% at the 30\% ratio, but this comes alongside a measurable decrease in model utility (ACC dropping to $\approx$ 57\%--67\%). Conversely, \textit{FLAME} achieves aggressive backdoor suppression (capping ASR at $\le 12.5\%$ across all settings), but this mitigation introduces a severe trade-off with the primary task, reducing the ACC to approximately 61\%--62\%.
	
In comparison, the topology-aware allocation mechanism (\textit{MAB(TOPY AWARE)}) substantially outperforms the random auditing baseline (\textit{MAB(RANDOM TOPY)}). At a 30\% malicious ratio, it reduces the ASR to 9.90\% in the scale-free topology and 19.29\% in the random-regular topology (compared to 28.77\% and 26.93\% for random placement, respectively). Meanwhile, it maintains a highly comparable main-task accuracy (ACC $\approx$ 68\%--71\%). This suggests that active, topology-guided auditing can successfully intercept structurally embedded backdoors without excessively penalizing benign statistical variance.
	
	\noindent \textbf{Computational Efficiency Analysis:} Figure \ref{fig:transformer_duration} illustrates the execution duration of various methods. Because our approach is executed on only 20\% of the nodes, it requires lower computational overhead compared to defense mechanisms that demand network-wide deployment.
\begin{table}[htbp]
	\centering
	\caption{Performance Comparison of Defense Mechanisms under Different Topologies (Transformer on PubMed)}
	\label{tab:transformer_results}
	\resizebox{\linewidth}{!}{
		\begin{tabular}{l *{6}{c}}
			\toprule
			\multirow{2}{*}{\textbf{Mechanism}} & \multicolumn{6}{c}{\textbf{Malicious Ratio}} \\
			\cmidrule(lr){2-7}
			& \multicolumn{2}{c}{10\%} & \multicolumn{2}{c}{20\%} & \multicolumn{2}{c}{30\%} \\
			\cmidrule(lr){2-3} \cmidrule(lr){4-5} \cmidrule(lr){6-7}
			& {ACC} & {ASR} & {ACC} & {ASR} & {ACC} & {ASR} \\
			\midrule
			
			\multicolumn{7}{l}{\textit{Topology: Scale-free}} \\ 
			\midrule
		FedAvg & 71.54 (01.75) & 39.33 (13.89) & 70.76 (01.58) & 54.17 (02.57) & 71.05 (01.61) & 59.43 (14.71) \\
		Krum & 71.76 (00.86) & 24.48 (10.31) & 71.78 (00.03) & 23.25 (10.94) & 70.52 (01.81) & 40.17 (10.56) \\
		TrimmedMean & 71.96 (00.64) & 05.51 (01.38) & 72.85 (00.53) & 23.85 (07.61) & 71.52 (00.88) & 46.05 (09.48) \\
		CosL2 & 62.88 (03.09) & 10.19 (05.06) & 63.69 (00.92) & 14.23 (03.90) & 66.53 (01.12) & 17.62 (04.18) \\
		FLAME & 61.88 (02.21) & 08.93 (01.44) & 62.38 (02.92) & 09.96 (00.92) & 62.12 (00.99) & 12.07 (03.19) \\
		MAB(RANDOM TOPY) & 68.38 (01.88) & 24.32 (09.92) & 69.55 (00.34) & 27.14 (06.50) & 69.06 (01.26) & 28.77 (08.93) \\
		MAB(TOPY AWARE) & 70.35 (00.48) & 03.08 (02.72) & 70.87 (00.89) & 08.89 (05.65) & 70.72 (00.95) & 09.90 (07.59) \\
			
			\midrule
			\addlinespace[0.5em] 
			
			\multicolumn{7}{l}{\textit{Topology: Random-regular}} \\ 
			\midrule
		FedAvg & 70.34 (00.59) & 60.96 (06.60) & 70.54 (00.92) & 71.76 (02.69) & 70.28 (00.66) & 75.40 (01.75) \\
		Krum & 70.43 (00.93) & 57.62 (07.04) & 71.51 (01.02) & 34.76 (18.96) & 70.69 (00.91) & 36.06 (04.69) \\
		TrimmedMean & 72.42 (00.43) & 04.24 (04.79) & 72.40 (00.33) & 22.56 (05.01) & 72.23 (00.44) & 40.32 (08.24) \\
		CosL2 & 57.13 (03.61) & 15.48 (06.51) & 60.06 (03.22) & 24.27 (05.39) & 66.86 (00.59) & 16.44 (03.73) \\
		FLAME & 61.91 (01.42) & 08.86 (01.92) & 62.05 (01.89) & 12.10 (03.41) & 62.34 (02.68) & 07.08 (01.47) \\
		MAB(RANDOM TOPY) & 69.52 (00.58) & 14.12 (05.06) & 68.99 (00.40) & 21.62 (06.74) & 70.13 (00.25) & 26.93 (11.56) \\
		MAB(TOPY AWARE) & 68.97 (00.68) & 08.42 (01.50) & 68.60 (01.05) & 12.55 (08.10) & 68.76 (00.20) & 19.29 (05.24) \\
			\bottomrule
			\addlinespace[0.5em]
			\multicolumn{7}{l}{\footnotesize \textit{Note:} Values in parentheses indicate the standard deviation across three independent trials.} \\
		\end{tabular}
	}
\end{table}
	\subsection{Experiment 2: Image Classification (CNN)}
	\label{sec:exp_cnn}
	
	To complete our comprehensive evaluation across diverse data modalities, we extend our framework to the continuous spatial domain of computer vision. 
	
	\subsubsection{Experimental Setup}
	
	\textbf{Dataset and Architecture:} We evaluate the system using the German Traffic Sign Recognition Benchmark (GTSRB) dataset~\citep{Stallkamp2012}, which consists of $43$ distinct traffic sign classes comprising a total of $39,209$ training images and $12,630$ test images. To simulate the decentralized environment, the training subset is partitioned across the participating clients. All images are resized to a standard $3 \times 32 \times 32$ RGB format. The global model utilizes a custom Convolutional Neural Network (CNN) architecture, comprising two convolutional layers (16 and 32 channels, respectively) followed by max-pooling, and two fully connected layers terminating in a 43-dimensional classification head. The decentralized network maintains the 20-client topology with a 20\% defense coverage.
	
	\textbf{Backdoor Attack Strategy (CNN):} 
	To compromise the visual classification task while evading Byzantine-robust aggregators, the adversary executes a three-phase poisoning strategy tailored for the CNN architecture.
	
	\textit{Phase 1: Raw Backdoor Generation}: 
The adversary constructs a poisoned local dataset $\tilde{\mathcal{S}}_i$. Specifically, a visible patch trigger is applied to $50\%$ of the images, where the trigger is defined as a localized spatial mask (e.g., a $4 \times 4$ pixel region at the bottom-right corner) with a predefined pixel intensity $\iota$ (e.g., $\iota = 4.0$).  Their corresponding labels are subsequently flipped to the target class $y_t = 7$. Using its preserved model state from the previous round, $\bm{\theta}_{mal}^{(t)}$, the adversary performs local gradient descent over $\tilde{\mathcal{S}}_i$ to obtain a raw malicious model state $\bm{\theta}_{poison}$. From the perspective of the defense mechanism, the implied raw malicious update is defined as:
$$
\Delta \bm{\theta}_{mal} = \bm{\theta}_{poison} - \bm{\theta}_{mal}^{(t)}.
$$

	\textit{Phase 2: Convex Vector Fusion (Direction Alignment)}: 
	Unlike the high-dimensional Transformer model that requires layer-wise asymmetric projections, the CNN's relatively compact parameter space allows for a global direction alignment strategy to bypass angle-based filters (e.g., Cosine Similarity). The adversary intercepts a benign reference update $\Delta \bm{\theta}_{ref}$ and computes a fused update $\Delta \tilde{\bm{\theta}}$ using a convex combination:
	$$
	\Delta \tilde{\bm{\theta}} = c_\alpha \cdot \Delta \bm{\theta}_{mal} + (1 - c_\alpha) \cdot \Delta \bm{\theta}_{ref},
	$$
	where the interpolation factor $c_\alpha = 0.6$ dictates the retention ratio of the backdoor intensity. This fusion geometrically aligns the malicious vector's direction with the benign consensus.
	
	\textit{Phase 3: Norm-Bounded Masking (Magnitude Alignment)}: 
	To guarantee evasion from Euclidean distance-based filters (e.g., Krum, Trimmed Mean), the adversary  limits the norm of the fused vector. By calculating the norm of the benign reference $N_{ref} = \|\Delta \bm{\theta}_{ref}\|$, the adversary sets a target norm bound $\tau = b_f\times N_{ref}$(e.g., $b_f=6$). The fused vector is then rescaled to ensure its magnitude is  smaller than the benign average:
	\begin{equation}
		s_{global} = \frac{\tau}{\|\Delta \tilde{\bm{\theta}}\| }
	\end{equation}
	Finally, the adversary constructs and transmits the highly camouflaged absolute model parameters to its neighbors:
	\begin{equation}
		\bm{\theta}_{malicious}^{(t+1)} = \bm{\theta}_{mal}^{(t)} + s_{global} \cdot \Delta \tilde{\bm{\theta}}
	\end{equation}
	This bounds the malicious update tightly within the benign update space, rendering traditional geometric auditing mechanisms blind to the injected backdoor mapping.
	
	\textbf{Implementation Details of Active Auditing Metrics:} 
	
	\noindent \textit{Stochastic Entropy Anomaly ($\rho_{SEA}$)}: The auditing process constructs a specialized probe set $\mathcal{X}_{\text{trap}}$ consisting of $B=3$ continuous spatial tensors, each matching the input dimension $C \times H \times W$. Here, $C$ represents the number of input channels (e.g., $C=3$ for RGB images), while $H$ and $W$ denote the spatial height and width of the tensor, respectively. Specifically, the probes include: (1) a high-frequency spatial checkerboard pattern alternating between extreme pixel intensities ($+3.0$ and $-3.0$); (2)uniform random noise generated via the sign function applied to an isotropic Gaussian distribution ($\text{sign}(\mathcal{N}(0, \mathbf{I})) \times 3.0$); and (3) a channel-level shift across individual color channels (e.g., $+3.0, -3.0, +3.0$).
	
	\noindent \textit{Randomized Smoothing KL-Divergence ($\rho_{RS}$)}: We decouple the neural network to isolate the latent feature vector $\bm{z} = \phi_{\boldsymbol{\theta}}(\bm{x})$ at a deep target layer (e.g., the  fully connected layer). Using a clean input $\bm{x}$(e.g., a standardized Gaussian tensor), we execute $B=16$ independent stochastic forward passes. In each iteration, we construct the perturbed feature $\tilde{z}^{(m)} = \bm{z} + \boldsymbol{\epsilon}^{(m)}$ by injecting isotropic Gaussian noise with a standard deviation of $\sigma_{rs} = 0.5$. The terminal classifier subsequently maps these features to the noisy predictive distributions $\bm{P}_{noisy}^{(m)}$. The final anomaly score $\rho_{RS}$ is computed based on the expected KL divergence between $\bm{P}_{clean}$ and $\bm{P}_{noisy}^{(m)}$.
	
	\noindent \textit{Activation Kurtosis ($\rho_{\mathcal{A}_K}$)}:  The probe set $\mathcal{X}_{probe}$ is constructed using a clean batch of $N=16$ genuine traffic sign images drawn from the GTSRB validation distribution. The input $\bm{r}^{(n)}$ corresponds to the vectorized spatial feature maps feeding into the penultimate fully connected layer, yielding the deep activation vector $\bm{z}^{(n)}$. 
	
	\begin{figure*}[t]
		\centering
		\begin{subfigure}{0.48\linewidth}
			\centering
			\includegraphics[width=\linewidth]{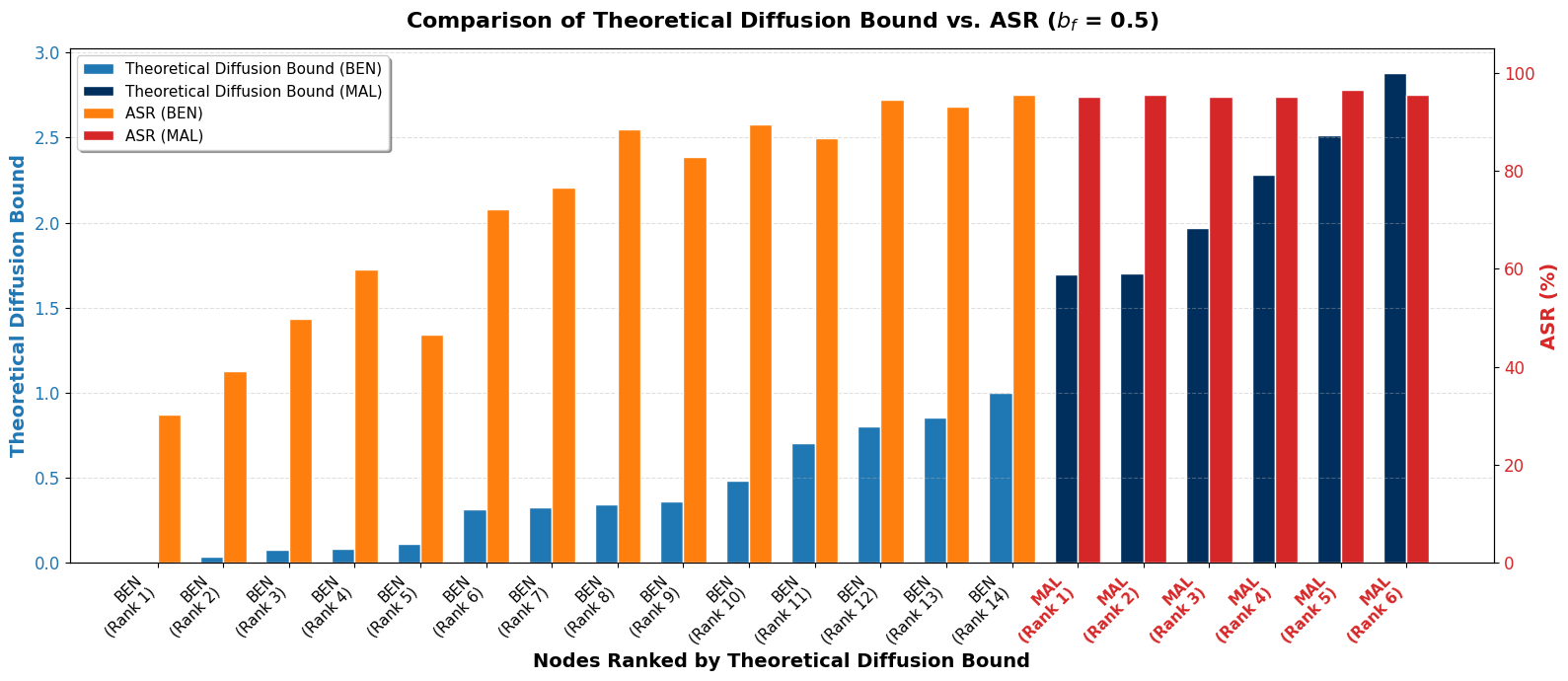}
			\caption{Theoretical Diffusion Bound vs. ASR \label{fig:cnn_thm}}
		\end{subfigure}
		\hfill
		\begin{subfigure}{0.48\linewidth}
			\centering
			\includegraphics[width=\linewidth]{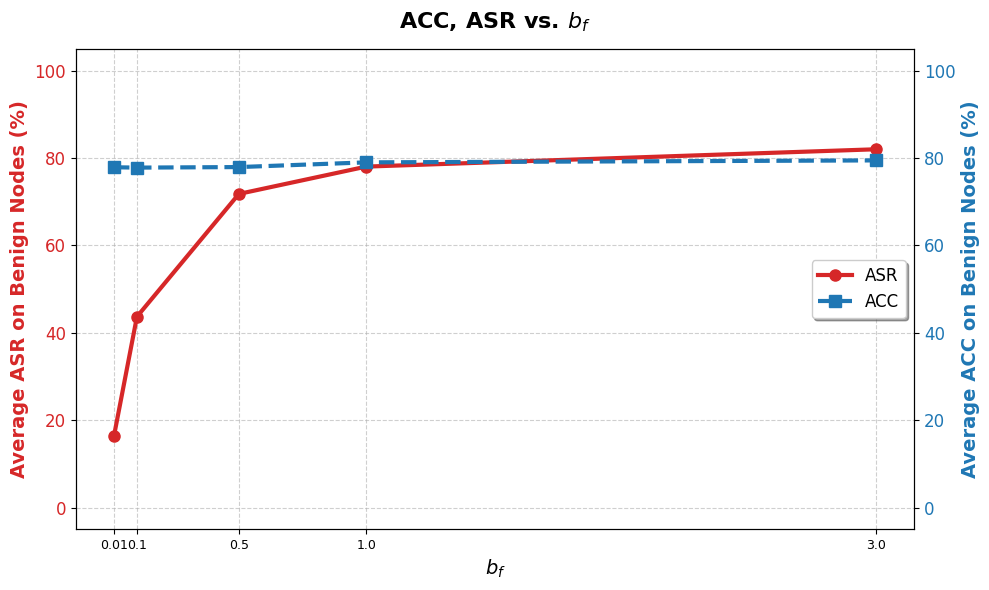}
			\caption{ACC and ASR vs. $b_f$ \label{fig:cnn_phase}}
		\end{subfigure}
		\caption{Diffusion model validation and Phase Transition phenomenon on the CNN architecture using the GTSRB dataset.}
		\label{fig:cnn_results}
	\end{figure*}
	\begin{table}[htbp]
		\centering
		\caption{Performance Comparison of Defense Mechanisms (CNN on GTSRB)}
		\label{tab:cnn_results}
		\resizebox{\linewidth}{!}{
			\begin{tabular}{l *{6}{c}}
				\toprule
				\multirow{2}{*}{\textbf{Mechanism}} & \multicolumn{6}{c}{\textbf{Malicious Ratio}} \\
				\cmidrule(lr){2-7}
				& \multicolumn{2}{c}{10\%} & \multicolumn{2}{c}{20\%} & \multicolumn{2}{c}{30\%} \\
				\cmidrule(lr){2-3} \cmidrule(lr){4-5} \cmidrule(lr){6-7}
				& {ACC} & {ASR} & {ACC} & {ASR} & {ACC} & {ASR} \\
				\midrule
				
				\multicolumn{7}{l}{\textit{Topology: Scale-free}} \\ 
				\midrule
		FedAvg & 83.75 (00.84) & 23.78 (13.41) & 82.31 (01.01) & 49.13 (20.50) & 78.67 (01.50) & 59.93 (06.23) \\
		Krum & 84.34 (00.86) & 24.54 (13.97) & 81.14 (02.47) & 43.21 (09.24) & 75.97 (01.16) & 67.21 (06.53) \\
		TrimmedMean & 83.66 (01.13) & 13.75 (05.78) & 82.23 (01.11) & 18.93 (13.27) & 78.30 (02.97) & 26.90 (05.87) \\
		CosL2 & 52.48 (03.43) & 14.73 (03.06) & 56.11 (05.38) & 11.29 (04.60) & 66.97 (01.07) & 20.34 (06.12) \\
		FLAME & 73.11 (00.21) & 12.70 (03.78) & 72.81 (00.21) & 13.45 (04.91) & 73.71 (01.08) & 09.88 (01.77) \\
		MAB(RANDOM TOPY) & 84.84 (01.26) & 05.67 (04.74) & 84.42 (02.72) & 06.87 (01.91) & 84.18 (02.81) & 12.72 (07.75) \\
		MAB(TOPY AWARE) & 84.24 (00.67) & 03.43 (01.36) & 85.00 (00.73) & 01.84 (01.50) & 85.03 (00.77) & 04.12 (02.54) \\
				\bottomrule
				
				\midrule
				\addlinespace[0.5em] 
				
				\multicolumn{7}{l}{\textit{Topology: Random-regular}} \\ 
				\midrule
		FedAvg & 84.50 (00.57) & 47.70 (04.19) & 80.99 (01.03) & 72.74 (18.26) & 79.38 (01.97) & 72.55 (21.87) \\
		Krum & 84.34 (00.56) & 49.20 (01.94) & 77.60 (00.55) & 69.62 (26.62) & 74.92 (02.53) & 76.51 (06.31) \\
		TrimmedMean & 84.09 (00.42) & 47.83 (05.09) & 79.80 (02.01) & 19.33 (14.56) & 78.77 (01.78) & 22.06 (18.16) \\
		CosL2 & 35.27 (10.80) & 15.19 (11.89) & 53.22 (13.37) & 18.12 (00.64) & 61.34 (14.48) & 17.99 (02.51) \\
		FLAME & 73.33 (00.56) & 16.48 (03.05) & 72.63 (00.47) & 15.58 (01.99) & 72.81 (00.49) & 13.68 (04.33) \\
		MAB(RANDOM TOPY) & 85.96 (01.05) & 04.12 (05.19) & 85.21 (00.94) & 04.68 (05.38) & 85.82 (01.35) & 11.13 (10.34) \\
		MAB(TOPY AWARE) & 85.97 (00.33) & 04.93 (06.84) & 84.89 (01.41) & 07.87 (04.88) & 85.70 (00.50) & 05.01 (05.66) \\
				\bottomrule
						\addlinespace[0.5em]
				\multicolumn{7}{l}{\footnotesize \textit{Note:} Values in parentheses indicate the standard deviation across three independent trials.} \\
			\end{tabular}
		}
	\end{table}
	
	\subsubsection{Results}
	
	\textbf{Validation of the Diffusion Model:} 	Figure~\ref{fig:cnn_results} illustrates the results for the scaling factor $b_f \in \{0.01,0.1,0.5,1,3\}$, with each scenario evaluated across three independent runs  of \textit{FedAvg} method. The average Kendall's tau correlation between the theoretical diffusion bound,  calculated using \eqref{eq:diffusion} with a signal attenuation of $\lambda_i=0.3$, and the ASR  is 0.68. Specifically, Figure~\ref{fig:cnn_thm} presents the averaged outcomes for $b_f = 0.5$, where the nodes are ranked according to their average theoretical diffusion bound. These findings demonstrate that our theoretical diffusion bound effectively captures the infection trajectory across the network. Furthermore, Figure~\ref{fig:cnn_phase} reveals a highly nonlinear relationship between the ASR and $b_f$. This suggests that an adversary requires only a marginal scaling factor $b_f$ to achieve a severe ASR, causing traditional magnitude-based detection mechanisms to lose their efficacy entirely. Throughout this parameter scaling, however, the ACC remains remarkably stable. This validates the extreme stealthiness of the injected backdoor, proving that it cannot be detected merely by monitoring for performance degradation on a clean validation set.
	
\noindent \textbf{Defense Mechanism Benchmark:}
Table \ref{tab:cnn_results} presents the performance of various defense mechanisms against the backdoor attack on the GTSRB dataset. The baseline results for \textit{FedAvg} confirm a severe security risk in decentralized vision tasks; the ASR reaches 59.93\% in Scale-free topologies and 72.55\% in Random-regular networks at a 30\% malicious ratio. This disparity suggests that the higher connectivity and lower diameter of random-regular graphs significantly accelerate the diffusion of malicious features.

Traditional Byzantine-robust aggregators exhibit inconsistent reliability in this configuration. \textit{Krum} exhibits degraded performance at the 30\% malicious ratio, where its ASR (67.21\%--76.51\%) actually exceeds that of \textit{FedAvg}, alongside a measurable decrease in main-task accuracy (ACC $\approx$ 75\%--76\%). \textit{TrimmedMean} also remains vulnerable at higher malicious ratios, with the ASR reaching 26.90\% in Scale-free graphs and 22.06\% in Random-regular settings. While \textit{FLAME} effectively caps the ASR below 17\% across all settings, it introduces a noticeable trade-off in utility, decreasing the ACC to approximately 72\%--74\%—a roughly 11\%--13\% reduction compared to the benign baseline. \textit{CosL2} struggles severely to maintain performance in this task, reducing the model's accuracy to levels close to random guessing (ACC $\approx$ 35\%--52\%) during the initial training stages under a 10\% malicious ratio.

In contrast, the \textit{MAB(TOPY AWARE)} method consistently maintains a highly competitive main-task accuracy of approximately 84\%--86\% across both tested topologies. Furthermore, at a 30\% malicious ratio, \textit{MAB(TOPY AWARE)} restricts the ASR to an exceptional 4.12\% in the scale-free topology and 5.01\% in the random-regular topology. This provides a substantial structural improvement over the \textit{MAB(RANDOM TOPY)} baseline, which allows ASRs of 12.72\% and 11.13\% under the same respective conditions. These results  indicate that integrating topological placement with active functional auditing can successfully neutralize camouflaged visual backdoors that otherwise bypass traditional geometric filters.
	
	\noindent \textbf{Computational Efficiency Analysis:} Figure \ref{fig:cnn_duration} illustrates the execution duration of various defense mechanisms. Because the parameter dimensionality of the CNN is substantially lower than that of the Transformer (0.27M vs. 67M parameters), the total computational overhead is dominated by the local training phase rather than the aggregation logic. Consequently, the computation times across all evaluated methods exhibit marginal differences.
	
\begin{table}[htbp]
	\centering
	\caption{Sensitivity Analysis of MAB Hyperparameters(Transformer on PubMed)}
	\label{tab:Transformer_sensitivity_split}
	\small
	\begin{tabular}{ccccc}
		\toprule
		\textbf{AUP} & \textbf{AP} & \textbf{AT} & \textbf{ACC (\%)} & \textbf{ASR (\%)} \\
		\midrule
		
		\midrule
		\multicolumn{5}{l}{\textit{Topology: Random-regular}} \\ 
		\midrule
		0.8 & 0.8 & 0.4 & 68.42 (01.18) & 21.58 (08.52) \\
		0.8 & 0.8 & 0.5 & 67.17 (01.36) & 19.67 (06.54) \\
		\cmidrule(lr){2-5}
		0.8 & 0.9 & 0.4 & 67.55 (00.79) & 17.94 (05.80) \\
		0.8 & 0.9 & 0.5 & 67.60 (01.93) & 23.53 (10.27) \\
		\cmidrule(lr){2-5}
		0.9 & 0.8 & 0.4 & 68.11 (01.81) & 19.81 (02.72) \\
		0.9 & 0.8 & 0.5 & 67.65 (01.75) & 18.67 (01.36) \\
		\cmidrule(lr){2-5}
		0.9 & 0.9 & 0.4 & 68.30 (01.24) & 20.12 (02.24) \\
		0.9 & 0.9 & 0.5 & 67.75 (01.62) & 21.85 (03.57) \\
		\cmidrule(lr){2-5}
		\midrule
		\multicolumn{5}{l}{\textit{Topology: Scale-free}} \\ 
		\midrule
		0.8 & 0.8 & 0.4 & 69.22 (00.79) & 09.70 (08.17) \\
		0.8 & 0.8 & 0.5 & 68.49 (00.98) & 12.73 (10.51) \\
		\cmidrule(lr){2-5}
		0.8 & 0.9 & 0.4 & 70.03 (00.65) & 09.16 (08.98) \\
		0.8 & 0.9 & 0.5 & 68.75 (00.89) & 09.28 (10.46) \\
		\cmidrule(lr){2-5}
		0.9 & 0.8 & 0.4 & 70.12 (00.97) & 14.03 (16.14) \\
		0.9 & 0.8 & 0.5 & 69.48 (00.47) & 09.33 (07.12) \\
		\cmidrule(lr){2-5}
		0.9 & 0.9 & 0.4 & 70.24 (00.97) & 11.98 (15.77) \\
		0.9 & 0.9 & 0.5 & 68.47 (00.71) & 08.06 (09.42) \\
		\cmidrule(lr){2-5}
		\bottomrule
	\end{tabular}
\end{table}
\begin{table}[htbp]
	\centering
	\caption{Sensitivity Analysis of MAB Hyperparameters(CNN on GTSRB)}
	\label{tab:CNN_sensitivity_split}
	\small
	\begin{tabular}{ccccc}
		\toprule
		\textbf{AUP} & \textbf{AP} & \textbf{AT} & \textbf{ACC (\%)} & \textbf{ASR (\%)} \\
		\midrule
		
		\midrule
		\multicolumn{5}{l}{\textit{Topology: Random-regular}} \\ 
		\midrule
		0.8 & 0.8 & 0.4 & 84.86 (01.07) & 03.37 (03.60) \\
		0.8 & 0.8 & 0.5 & 85.89 (00.97) & 05.55 (04.77) \\
		\cmidrule(lr){2-5}
		0.8 & 0.9 & 0.4 & 84.86 (01.07) & 03.37 (03.60) \\
		0.8 & 0.9 & 0.5 & 85.89 (00.97) & 05.55 (04.77) \\
		\cmidrule(lr){2-5}
		0.9 & 0.8 & 0.4 & 85.70 (00.50) & 05.01 (05.66) \\
		0.9 & 0.8 & 0.5 & 85.34 (01.16) & 10.67 (08.50) \\
		\cmidrule(lr){2-5}
		0.9 & 0.9 & 0.4 & 85.70 (00.50) & 05.01 (05.66) \\
		0.9 & 0.9 & 0.5 & 85.34 (01.16) & 10.67 (08.50) \\
		\cmidrule(lr){2-5}
		\midrule
		\multicolumn{5}{l}{\textit{Topology: Scale-free}} \\ 
		\midrule
		0.8 & 0.8 & 0.4 & 84.78 (00.82) & 07.23 (05.54) \\
		0.8 & 0.8 & 0.5 & 86.30 (00.50) & 04.35 (01.61) \\
		\cmidrule(lr){2-5}
		0.8 & 0.9 & 0.4 & 84.30 (00.57) & 04.71 (03.86) \\
		0.8 & 0.9 & 0.5 & 86.56 (00.88) & 03.78 (02.32) \\
		\cmidrule(lr){2-5}
		0.9 & 0.8 & 0.4 & 85.03 (00.77) & 04.12 (02.54) \\
		0.9 & 0.8 & 0.5 & 86.00 (00.56) & 07.34 (07.43) \\
		\cmidrule(lr){2-5}
		0.9 & 0.9 & 0.4 & 84.82 (00.57) & 11.70 (09.45) \\
		0.9 & 0.9 & 0.5 & 86.44 (00.42) & 06.23 (03.21) \\
		\cmidrule(lr){2-5}
		\bottomrule
	\end{tabular}
\end{table}
	
\subsection{MAB hyperparameters sensitivity analysis}
Across both Experiments 1 and 2, the global training process spans 15 communication rounds. To manage the exploration-exploitation trade-off during UCB selection, the constant is set to $c = 0.5$. Balancing computational overhead with high detection coverage, defense nodes audit 90\% of their neighbors and apply an 80\% aggregation subsampling rate. The aggregation threshold is set to $\tau_{agg} = 0.4$. Finally, the pre-allocated stratified weights $W_{self}$, $\tilde{W}_{ij}$ and $\tilde{W}^*_{ij}$ are configured as 0.5, 0.45, and 0.05, respectively.

To further understand the robustness of the proposed framework, we conducted a sensitivity analysis on three key hyperparameters in Algorithm~\ref{alg:MAB}: audit sample size rate (ADR), aggregation sample size rate (AGR), and $\tau_{agg}$. The results are summarized in Table \ref{tab:Transformer_sensitivity_split} (Transformer on PubMed) and Table \ref{tab:CNN_sensitivity_split} (CNN on GTSRB).

Across both tasks and network topologies, the main task accuracy remains highly stable. For the Transformer model, ACC fluctuates within a narrow margin of approximately 67\% to 71\%, while the CNN model maintains an ACC between 84\% and 87\%. This demonstrates that the MAB-guided aggregation and active auditing mechanisms do not severely degrade benign model utility, regardless of the specific hyperparameter configuration.

Regarding defensive efficacy, the results suggest that hyperparameter variations tend to have a relatively limited impact on the overall ASR, particularly when evaluated within the same network topology. For instance, when evaluated under a consistent topology, the ASR for the high-dimensional Transformer model exhibits relative stability. Specifically, global variations are confined to a moderate range, spanning from 8.06\% to 14.03\% in scale-free networks and from 17.94\% to 21.85\% in random-regular configurations. Similarly, in the CNN architecture, altering the parameters seems to shift the ASR only within a tight band of 3.37\% to 11.70\%. This implies that while precise tuning could potentially optimize the defense further, achieving a reliable baseline suppression capability might not strictly depend on finding an absolutely perfect parameter pairing, as the mechanism demonstrates a degree of inherent resilience under a fixed graph structure.
	
	\section{Discussion}
	This paper introduces an active auditing framework for Decentralized Federated Learning (DFL), shifting the defensive paradigm from passive filtering to proactive intervention. By formulating a  dynamical model, we effectively characterize the spatiotemporal diffusion of attacks and identify critical phase transitions in network contamination. Our multi-scale auditing pipeline, comprising $\rho_{SEA}$, $\rho_{RS}$, and $\rho_{\mathcal{A}_K}$, leverages information asymmetry to expose latent backdoors that are invisible to static geometric metrics.Combined with MAB-guided neighbor selection and topology-aware placement, the proposed framework achieves a superior balance between primary task accuracy and adversarial resilience. Theoretical analysis and extensive empirical validation across CNN and Transformer architectures demonstrate that our approach maintains robust convergence while substantially suppressing stealthy, adaptive backdoors.
	
	While the proposed framework and state-of-the-art defense algorithms, such as FLAME, CosL2, demonstrate substantial resilience under some attack models, they are not immune to the  challenges of data heterogeneity. In environments characterized by extreme Non-IID distributions, there exists a critical threshold where statistical diversity becomes indistinguishable from adversarial manipulation.
	Once data heterogeneity exceeds this limit, defense mechanisms struggle to differentiate between extreme non-IID benign updates and actual malicious gradients. In this regime, the system loses the discriminative power required to reliably isolate malicious gradients from legitimate, yet extreme, statistical outliers. Consequently, the very metrics designed to preserve global integrity may inadvertently penalize unique local features, highlighting an inherent trade-off between Byzantine robustness and the preservation of client-specific intelligence.
		\section{Data Availability Statement}
	The source code and scripts required to reproduce all numerical results and figures presented in this paper are publicly available in the following GitHub repository: \url{https://github.com/Sheng-Pan/Decentralized-federated-learning2}.
	\appendix
	\section{Proofs}
	\subsection{Proofs of section \ref{sec:diffu}}
	\begin{lemma}\label{lem:a1}
		Suppose that the network graph  $ G $  induced by  $ \mathbf{W} $  is \textbf{strongly connected} and there exists at least one benign node  $ i $  with  $ \Lambda_{ii} = \lambda > 0 $, then the spectral radius of  $ \mathbf{A} $  satisfies  $ \rho(\mathbf{A}) < 1 $ .
	\end{lemma}
	
	\begin{proof}
		By the property of $\mathbf{W}$ and $ \bf{\Lambda}$, we have that
		$ \mathbf{A} $  is non-negative since  $ \mathbf{W} $  is non-negative and  $ 0 \leq 1- \Lambda_{ii} \leq 1 $ ,
		$ \mathbf{A} $  is irreducible because the graph corresponding to  $ \mathbf{W} $  is strongly connected and
		the row sums of  $ \mathbf{A} $  are:
		$$
		\sum_j A_{ij} = 
		\begin{cases}
			1-\lambda < 1 & \text{if node } i \text{ is benign}\\
			1 & \text{if node } i \text{ is malicious}.
		\end{cases}
		$$
		We use a  proof by contradiction method.
		Now	assume  $ \rho(\mathbf{A}) = 1 $ . Since  $ \mathbf{A} $  is a non-negative irreducible matrix, by the Perron-Frobenius theorem(e.g \citep{horn2012matrix}),
		there exists a unique positive eigenvector  $ \mathbf{v} > \mathbf{0} $  (all components  positive).
		This eigenvector corresponds to the eigenvalue  $ \rho(\mathbf{A}) = 1 $ .
		That is,  $ \mathbf{A}\mathbf{v} = \mathbf{v} $ .
		
		Let  $ i $  be any benign node. From  $ \mathbf{A}\mathbf{v} = \mathbf{v} $ , we have:

		$$
		v_i = \sum_j A_{ij}v_j = (1-\lambda)\sum_j W_{ij}v_j.
		$$
		Since  $ \lambda > 0 $  and  $ W_{ij} \geq 0 $ , it follows that:
		\begin{align}\label{eq:a.1}
			v_i = (1-\lambda)\sum_j W_{ij}v_j < \sum_j W_{ij}v_j.
		\end{align}
		Let  $ v_{\max} = \max_k v_k > 0 $ , and let  $ S = \{k: v_k = v_{\max}\} $  be the set of nodes achieving this maximum value.
		We now show that
		$ S $  cannot contain any benign nodes.
		
		If there exists a benign node  $ i \in S $ , then  $ v_i = v_{\max} $ . However, from \eqref{eq:a.1}:

		$$
		v_i = (1-\lambda)\sum_j W_{ij}v_j \leq (1-\lambda)v_{\max} < v_{\max}.
		$$

		This contradicts  $ v_i = v_{\max} $ . Therefore,  $ S $  contains only malicious nodes.

		Since graph  $ G $  is strongly connected, for any  $ k \in S $  (a malicious node) and any benign node  $ i $ , there exists a directed path  $ k \to j_1 \to j_2 \to \cdots \to i $ .
		
		Consider  $ \mathbf{A}^m\mathbf{v} = \mathbf{v} $ , where  $ m $  is the path length. Specifically, for node  $ k \in S $ :

		$$
		v_k = (\mathbf{A}^m\mathbf{v})_k = \sum_{\ell} (\mathbf{A}^m)_{k\ell}v_\ell.
		$$

		Due to strong connectivity,  $ (\mathbf{A}^m)_{ki} > 0 $  (since there exists a path from  $ k $  to  $ i $ , and the product of weights along this path is positive).
		
		Since  $ i $  is a benign node, from  \eqref{eq:a.1} we know  $ v_i < v_{\max} $ . Additionally,  $ v_\ell \leq v_{\max} $  for all  $ \ell $ . Therefore:

		$$
		v_k = \sum_{\ell} (\mathbf{A}^m)_{k\ell}v_\ell \leq (\mathbf{A}^m)_{ki}v_i + \sum_{\ell \neq i} (\mathbf{A}^m)_{k\ell}v_{\max}.
		$$

		Since  $ (\mathbf{A}^m)_{ki} > 0 $  and  $ \sum_{\ell} (\mathbf{A}^m)_{k\ell} \leq 1 $  (because the row sums of  $ \mathbf{A} $  are at most 1), we have:

		$$
		v_k < (\mathbf{A}^m)_{ki}v_{\max} + \sum_{\ell \neq i} (\mathbf{A}^m)_{k\ell}v_{\max} = v_{\max}\sum_{\ell} (\mathbf{A}^m)_{k\ell} \leq v_{\max}.
		$$

		This contradicts  $ v_k = v_{\max} $ .
		
		Therefore, the assumption  $ \rho(\mathbf{A}) = 1 $  is false. Since  $ \mathbf{A} $  is non-negative, its spectral radius is non-negative, and from  $ \|\mathbf{A}\|_\infty \leq 1 $  we know  $ \rho(\mathbf{A}) \leq 1 $ . Hence, we must have:

		$$
		\rho(\mathbf{A}) < 1.
		$$

		This completes the proof.
	\end{proof}
	
	\textbf{Proof of Lemma \ref{lem:1}:}
	
	Starting from $\boldsymbol{s}	(0) = \mathbf{0}$, we iteratively expand:
	\[
	\begin{aligned}
		\boldsymbol{s}	(1) &= \mathbf{u}, \\
		\boldsymbol{s}	(2) &= \mathbf{A} \mathbf{u} + \mathbf{u} = (\mathbf{I} + \mathbf{A}) \mathbf{u}, \\
		\boldsymbol{s}	(3) &= \mathbf{A} (\mathbf{A} \mathbf{u} + \mathbf{u}) + \mathbf{u} = (\mathbf{I} + \mathbf{A} + \mathbf{A}^2) \mathbf{u}.
	\end{aligned}
	\]
	By induction,
	\[
	\boldsymbol{s}	(t+1) = \sum_{k=0}^{t} \mathbf{A}^k \mathbf{u}.
	\]
	By Lemma \ref{lem:a1}, by the matrix geometric series identity:
	\[
	\sum_{k=0}^{t} \mathbf{A}^k = (\mathbf{I} - \mathbf{A}^{t+1})(\mathbf{I} - \mathbf{A})^{-1},
	\]
	and noting that $\boldsymbol{s}	^* = (\mathbf{I} - \mathbf{A})^{-1} \mathbf{u}$, we obtain:
	\[
	\boldsymbol{s}	(t+1) = (\mathbf{I} - \mathbf{A}^{t+1}) (\mathbf{I} - \mathbf{A})^{-1} \mathbf{u} = (\mathbf{I} - \mathbf{A}^{t+1}) \boldsymbol{s}	^*.
	\]
	This completes the proof. \qed
	
	\textbf{Proof of Lemma \ref{cor:1}:}
	
	According to the linear system dynamics equation with initial state $\boldsymbol{s}(0)=\mathbf{0}$, the solution is:
	
	$$
	\boldsymbol{s}(t) = \sum_{k=0}^{t-1} \mathbf{A}^k \mathbf{u}.
	$$
	
For a specific node $i$, utilizing the component-wise upper bound $\mathbf{A} \le (1-\lambda)\mathbf{W}$:
$$
s_i(t) \le u_s \sum_{k=0}^{t-1} (1-\lambda)^k (\mathbf{W}^k)_{is}.
$$
	
	where $(\mathbf{W}^k)_{is}$ represents the total transition probability from source $s$ to node $i$ in exactly $k$ steps.
	
	$(\mathbf{W}^k)_{is} > 0$ if and only if there exists a path of length $k$ between $s$ and $i$. Since $d_{is}$ is the shortest path distance, for any $k < d_{is}$, no path exists between the two nodes:
	
	$$
	(\mathbf{W}^k)_{is} = 0, \quad \forall k < d_{is}.
	$$
	
	Therefore, the lower bound of summation changes from $0$ to $d_{is}$:
	
	$$
	s_i(t) = u_s \sum_{k=d_{is}}^{t-1} (1-\lambda)^k (\mathbf{W}^k)_{is}.
	$$

	Let $m = k - d_{is}$ (representing ``redundant steps'' beyond the shortest path), substituting into the above equation:
	
	$$
	s_i(t) = u_s \sum_{m=0}^{t-1-d_{is}} (1-\lambda)^{m + d_{is}} (\mathbf{W}^{m+d_{is}})_{is}.
	$$
	
	Extracting the constant term $(1-\lambda)^{d_{is}}$:
	
	$$
	s_i(t) = u_s \cdot (1-\lambda)^{d_{is}} \cdot \sum_{m=0}^{t-1-d_{is}} (1-\lambda)^m (\mathbf{W}^{m+d_{is}})_{is}.
	$$

	Since $\mathbf{W}$ is a row-stochastic matrix, all its elements are $\leq 1$, thus
	
	$$
	s_i(t) \leq u_s \cdot (1-\lambda)^{d_{is}} \cdot \sum_{m=0}^{t-1-d_{is}} (1-\lambda)^m.
	$$
	
	\qed
	
	\subsection{Proofs of section \ref{sec:4}}
	\begin{lemma}[Non-Convex Convergence of Nominal Sequence]\label{lem:a3}
		Under Assumptions \ref{ass:1}-\ref{ass:3}, the nominal sequence $\hat{\bm{\theta}}_i^{(t)}$ follows standard decentralized SGD updates without any Byzantine interference. For a non-convex objective function $F$, let $\rho_{\text{nom}} = |\lambda_2(\mathbf{W}^{(t)})| < 1$ denote the nominal spectral radius governing the diffusion of benign gradients. With a learning rate $\eta = \mathcal{O}(1/\sqrt{T})$, the average expected squared gradient norm of the nominal global model $\bar{\hat{\bm{\theta}}}^{(t)} = \frac{1}{n}\sum_{i=1}^n \hat{\bm{\theta}}_i^{(t)}$ over $T$ iterations is bounded by:
		$$
		\frac{1}{T}\sum_{t=0}^{T-1}\mathbb{E}[\|\nabla F(\bar{\hat{\bm{\theta}}}^{(t)})\|^2] \leq \mathcal{O}\left(\frac{1}{\sqrt{nT}}\right) + \mathcal{O}\left(\frac{1}{T(1-\rho_{\text{nom}})^2}\right).
		$$
	\end{lemma}

	\begin{proof}
		This result follows directly from the standard convergence analysis for decentralized stochastic gradient descent on non-convex objectives, as  established by \citet{lian2017can}. 
		
		Since the nominal sequence $\hat{\bm{\theta}}_i^{(t)}$ is defined as the counterfactual trajectory where all participating nodes behave honestly, the system reduces exactly to a standard decentralized optimization problem over the nominal topology $\mathbf{W}^{(t)}$. By setting the learning rate $\eta = \mathcal{O}(1/\sqrt{T})$, the detailed convergence bound separates into the asymptotic optimization rate $\mathcal{O}(1/\sqrt{nT})$ and the transient network consensus penalty $\mathcal{O}(1/(T(1-\rho_{\text{nom}})^2))$.
	\end{proof}

	\begin{lemma}[Boundedness of Backdoor Bias with Topological Heterogeneity]\label{lem:a4}
		Under Assumptions \ref{ass:1}-\ref{ass:6}, the global backdoor deviation $\boldsymbol{E}	^{(t)}$ satisfies:
		$$
		\|\boldsymbol{E}	^{(t)}\|^2 \leq (\bm{\delta}^{(t)})^2,
		\label{eq:bias_bound_topo}
		$$
		where $\bm{\delta}^{(t)} = \rho_{\text{err}}^t \|\boldsymbol{E}	^{(0)}\| + U \sum_{k=0}^{t-1} \rho_{\text{err}}^k$, and $U = \sup_{t \ge 0} \|\tilde{\mathbf{M}	}^{(t)}\|$ is the maximum malicious source energy bounded by Assumption \ref{ass:6}. 
		
		Furthermore, if the learning rate and topology satisfy the stability condition $\rho_{\text{err}} < 1$, then asymptotically:
		$$
		\|\boldsymbol{E}	^{(t)}\|^2 \leq \left( \frac{U}{1-\rho_{\text{err}}}\right)^2.
		$$
	\end{lemma}
	
\begin{proof}
	We track the trajectory divergence between the contaminated process and the ideal clean process.  The realization of the global error state evolves as a discrete-time dynamical system driven by an exogenous adversarial shock:
	$$
	\boldsymbol{E}	^{(t+1)} = (\mathbf{W}^{(t)} \otimes \mathbf{I}_p) \left( \boldsymbol{E}	^{(t)} - \eta \Delta \mathbf{G}_t \right) + \tilde{\mathbf{M}	}^{(t)}.
	$$
	
	Under Assumption \ref{ass:1}, the objective function has Lipschitz continuous gradients. This ensures that a single optimization step linearly bounds the error amplification of the state error:
	$$
	\|\boldsymbol{E}	^{(t)} - \eta \Delta \mathbf{G}_t\| \leq \|\boldsymbol{E}	^{(t)}\| + \eta L \|\boldsymbol{E}	^{(t)}\| = (1 + \eta L) \|\boldsymbol{E}	^{(t)}\| .
	$$
	
	To analyze the systemic risk, we stratify the network into two mutually exclusive sub-populations: the defending stratum $\mathcal{D}$ and the vulnerable stratum $\mathcal{V}$. Partitioning the global error vector accordingly, $\boldsymbol{E}	^{(t)} = [(\boldsymbol{E}	_{\mathcal{D}}^{(t)})^\top, (\boldsymbol{E}	_{\mathcal{U}}^{(t)})^\top]^\top$, we define the intra-stratum error propagation rates: $\rho_{\mathcal{D}} := \|\mathbf{w}_{\mathcal{DD}}^{(t)}\|(1 + \eta L)$ and $\rho_{\mathcal{U}} := \|\mathbf{w}_{\mathcal{VV}}^{(t)}\|(1 + \eta L)$. Taking the marginal $\ell_2$-norms yields a coupled bivariate system of inequalities:
	\begin{align*}
		\|\boldsymbol{E}	_{\mathcal{D}}^{(t+1)}\| &\leq \rho_{\mathcal{D}}\|\boldsymbol{E}	_{\mathcal{D}}^{(t)}\| + \|\mathbf{w}_{\mathcal{DU}}^{(t)}\|(1+\eta L)\cdot\|\boldsymbol{E}	_{\mathcal{U}}^{(t)}\|, \\
		\|\boldsymbol{E}	_{\mathcal{U}}^{(t+1)}\| &\leq \|\mathbf{w}_{\mathcal{UD}}^{(t)}\|(1+\eta L)\cdot\|\boldsymbol{E}	_{\mathcal{D}}^{(t)}\| + \rho_{\mathcal{U}}\|\boldsymbol{E}	_{\mathcal{U}}^{(t)}\| + U.
	\end{align*}
	
	This  bounding process can be compactly represented by the transition operator of the coupled system:
	$$
	\mathbf{w}^{(t+1)} \leq \mathbf{T} \mathbf{w}^{(t)} + \mathbf{b}	, \quad \text{where } \mathbf{T} = 
	\begin{bmatrix}
		\rho_{\mathcal{D}} & \|\mathbf{w}_{\mathcal{DU}}^{(t)}\|(1+\eta L) \\
		\|\mathbf{w}_{\mathcal{UD}}^{(t)}\|(1+\eta L) & \rho_{\mathcal{U}}
	\end{bmatrix},
	$$
	with the marginal norm vector $\mathbf{w}^{(t)} = [\|\boldsymbol{E}	_{\mathcal{D}}^{(t)}\|, \|\boldsymbol{E}	_{\mathcal{U}}^{(t)}\|]^\top$ and the worst-case shock vector $\mathbf{b}	 = [0, U]^\top$. 
	
	The stability  of this coupled system is governed by the eigenvalues of the transition matrix $\mathbf{T}$. By Lemma 5.6.10 of \citep{horn2012matrix}, we can bound the operator norm of $\mathbf{T}$ by its spectral radius (the maximum absolute eigenvalue) plus an arbitrarily small constant $\epsilon > 0$. Thus, there exists a valid matrix norm $\|\cdot\|_{\text{mat}}$ such that $\|\mathbf{T}\|_{\text{mat}} \leq \rho(\mathbf{T}) + \epsilon \leq \rho_{\text{err}} + \epsilon$. Mapping our state vector into this norm space via $e^{(t)} = \|\mathbf{w}^{(t)}\|_{\text{mat}}$, the system relaxes to a scalar autoregressive bound: $e^{(t+1)} \leq \rho_{\text{err}} e^{(t)} + \|\mathbf{b}	\|_{\text{mat}}$.
	
	Unrolling this recursive AR(1) filter from initialization $t=0$, the accumulation of the adversarial shocks forms a geometric progression:
	$$
	\|\boldsymbol{E}	^{(t)}\| \leq \rho_{\text{err}}^t \|\boldsymbol{E}	^{(0)}\| + U \sum_{k=0}^{t-1} \rho_{\text{err}}^k .
	$$
	Squaring both sides allows us to bound the maximum expected energy of the error process:
	$$
	\|\boldsymbol{E}	^{(t)}\|^2 \leq \left( \rho_{\text{err}}^t \|\boldsymbol{E}	^{(0)}\| + U \sum_{k=0}^{t-1} \rho_{\text{err}}^k \right)^2.
	$$
	
	Finally, the stability condition $\rho_{\text{err}} < 1$ is  equivalent to ensuring our autoregressive error process is  stationary (its roots lie inside the unit circle). The geometric series of accumulated shocks converges, yielding a finite asymptotic upper bound as $t \to \infty$, concluding the proof.
\end{proof}
	\textbf{Proof of Theorem \ref{thm:2}:}
	
	To evaluate the asymptotic behavior of the empirical risk minimizer, we bound the expected squared gradient norm of the contaminated sample mean $\bar{\bm{\theta}}^{(t)} = \frac{1}{n}\sum_{i=1}^n \bm{\theta}_i^{(t)}$. We introduce the counterfactual (uncontaminated) stochastic process $\bar{\hat{\bm{\theta}}}^{(t)}$ as a control.
	
	Using the basic quadratic inequality $\|a + b\|^2 \leq 2\|a\|^2 + 2\|b\|^2$, we bound the gradient norm by partitioning it into the baseline variance of the control process and the error induced by the adversarial shocks:
	$$
	\|\nabla F(\bar{\bm{\theta}}^{(t)})\|^2 \leq 2\|\nabla F(\bar{\hat{\bm{\theta}}}^{(t)})\|^2 + 2\|\nabla F(\bar{\bm{\theta}}^{(t)}) - \nabla F(\bar{\hat{\bm{\theta}}}^{(t)})\|^2.
	$$
	
	By Assumption \ref{ass:1}, the objective function has Lipschitz continuous gradients, thus 
	$$
	\|\nabla F(\bar{\bm{\theta}}^{(t)}) - \nabla F(\bar{\hat{\bm{\theta}}}^{(t)})\|^2 \leq L^2 \|\bar{\bm{\theta}}^{(t)} - \bar{\hat{\bm{\theta}}}^{(t)}\|^2 = L^2 \|\bar{\boldsymbol{e}	}^{(t)}\|^2,
	$$
	where $\bar{\boldsymbol{e}	}^{(t)} = \frac{1}{n}\sum_{i=1}^n \boldsymbol{e}	_i^{(t)}$ is the cross-sectional mean of the adversarial shock accumulations. Taking the unconditional expectation and time-averaging over $T$ periods yields:
	$$
	\frac{1}{T}\sum_{t=0}^{T-1}\mathbb{E}[\|\nabla F(\bar{\bm{\theta}}^{(t)})\|^2] \leq \underbrace{\frac{2}{T}\sum_{t=0}^{T-1}\mathbb{E}[\|\nabla F(\bar{\hat{\bm{\theta}}}^{(t)})\|^2]}_{\text{Term A: Counterfactual Baseline}} + \underbrace{\frac{2L^2}{T}\sum_{t=0}^{T-1}\mathbb{E}[\|\bar{\boldsymbol{e}	}^{(t)}\|^2]}_{\text{Term B: Systemic Risk Bound}}.
	$$
	
	For Term A, according to Lemma \ref{lem:a3}, applying a decaying learning rate schedule $\eta = \mathcal{O}(1/\sqrt{T})$ decomposes the variance of the uncontaminated process into the standard asymptotic rate and a finite-sample spatial bias driven by the nominal transition matrix $\rho_{\text{nom}}$:
	$$
	\frac{1}{T}\sum_{t=0}^{T-1}\mathbb{E}[\|\nabla F(\bar{\hat{\bm{\theta}}}^{(t)})\|^2] \leq \mathcal{O}\left(\frac{1}{\sqrt{nT}}\right) + \mathcal{O}\left(\frac{1}{T(1-\rho_{\text{nom}})^2}\right).
	$$
	
	For Term B, which captures the systemic propagation of errors, we apply Jensen's Inequality to tightly bound the cross-sectional mean by the trace of the system's spatial variance vector $\boldsymbol{E}	^{(t)}$:
	$$
	\|\bar{\boldsymbol{e}	}^{(t)}\|^2 = \left\| \frac{1}{n}\sum_{i=1}^n \boldsymbol{e}	_i^{(t)} \right\|^2 \leq \frac{1}{n}\sum_{i=1}^n \|\boldsymbol{e}	_i^{(t)}\|^2 = \frac{1}{n} \|\boldsymbol{E}	^{(t)}\|^2.
	$$

	By Lemma \ref{lem:a4}, provided the stationarity condition $\rho_{\text{err}} < 1$ holds, the maximum variance is  bounded by the steady-state limit of the AR(1) process:
	$$
	\|\bar{\boldsymbol{e}	}^{(t)}\|^2 \leq \left( \frac{U}{1-\rho_{\text{err}}}\right)^2.
	$$
	
	Note that the spatial smoothing matrix $(\mathbf{W}^{(t)} \otimes \mathbf{I}_p)$ is Markovian (row-stochastic), implying $\|\mathbf{W}^{(t)} \otimes \mathbf{I}_p\|_2 \leq 1$. Under Assumption \ref{ass:6}, the exogenous adversarial shocks are uniformly bounded by $C_u$. Therefore, the total exogenous shock variance is bounded by the proportion of contaminated strata:
	$$
	U^2 = \sup_{t \ge 0} \|\tilde{\mathbf{M}	}^{(t)}\|^2 \leq \sup_{t \ge 0} \|\mathbf{M}	^{(t)}\|^2 \le \sup_{t \ge 0} \sum_{j \in \mathcal{M}} \|\mathbf{m}	_j^{(t)}\|^2  = \mathcal{O}(|\mathcal{M}| C_u^2).
	$$
	
	Substituting the total shock variance $U^2$ back into the stationary limit for Term B yields the expected systemic bias:
	$$
	\|\bar{\boldsymbol{e}	}^{(t)}\|^2  \leq \mathcal{O}\left( \frac{|\mathcal{M}| C_u^2}{n(1-\rho_{\text{err}})^2} \right) = \mathcal{O}\left( \frac{q_m C_u^2}{(1-\rho_{\text{err}})^2} \right).
	$$
	
	Finally, combining the asymptotic bounds for the uncontaminated baseline (Term A) and the stationary adversarial bias (Term B) yields the final expected risk bound:
	$$
	\begin{aligned}
		\frac{1}{T}\sum_{t=0}^{T-1}\mathbb{E}[\|\nabla F(\bar{\bm{\theta}}^{(t)})\|^2] \leq & \underbrace{\mathcal{O}\left(\frac{1}{\sqrt{nT}}\right)}_{\text{Standard Asymptotic Rate}} + \underbrace{\mathcal{O}\left(\frac{1}{T(1-\rho_{\text{nom}})^2}\right)}_{\text{Consensus error}} \\
		& + \underbrace{\mathcal{O}\left( \frac{q_m C_u^2}{(1-\rho_{\text{err}})^2}\right)}_{\text{Asymptotic Adversarial Error}}.
	\end{aligned}
	$$
	This completes the proof. \qed
	
	\bibliography{ref.bib}
\end{document}